\documentclass[10pt,journal,letterpaper,compsoc]{IEEEtran_new}
\newcommand{\ttn}{\mathbf{\tilde{n}}}
\newcommand{\tn}{\mathbf{n}}
\newcommand{\bx}[0]{{\mathbf{x}}}
\newcommand{\bv}[0]{{\mathbf{h}}}   %% change v to h
\newcommand{\bK}[0]{{\mathbf{K}}}
\newcommand{\bS}[0]{{\mathbf{S}}}
\newcommand{\bG}[0]{{\mathbf{G}}}
\newcommand{\bO}[0]{{\mathbf{O}}}
\newcommand{\bR}[0]{{\mathbf{R}}}
\newcommand{\bM}[0]{{\mathbf{M}}}
\newcommand{\bB}[0]{{\mathbf{B}}}
\newcommand{\bE}[0]{{\mathbf{E}}}
\newcommand{\bI}[0]{{\mathbf{I}}}
\newcommand{\bb}[0]{{\mathbf{b}}}
\newcommand{\bX}[0]{{\mathbf{X}}}
\newcommand{\bff}[0]{{\mathbf{f}}}
\newcommand{\br}{\mathbf{r}}
\newcommand{\bT}[0]{{\mathbf{T}}}
\newcommand{\bN}{\mathbf{N}}
\newcommand{\bA}[0]{{\mathbf{A}}}

\newtheorem{theorem}{Theorem}

\newtheorem{corollary}{Corollary}

\newtheorem{proof}[theorem]{Proof}

\newcommand{\qed}{\nobreak \ifvmode \relax \else
      \ifdim\lastskip<1.5em \hskip-\lastskip
      \hskip1.5em plus0em minus0.5em \fi \nobreak
      \vrule height0.75em width0.5em depth0.25em\fi}

\ifCLASSOPTIONcompsoc
\else
\fi
\ifCLASSINFOpdf
\else
\fi
\usepackage{amssymb}
\usepackage{graphicx}
\usepackage{epsfig}
\usepackage{psfrag}

\usepackage[cmex10]{amsmath}
\begin{document}
\title{A Closed-Form Solution to Tensor Voting: Theory and Applications}

\author{Tai-Pang~Wu,~Sai-Kit~Yeung,~Jiaya~Jia,~Chi-Keung~Tang,~\and
G\'{e}rard Medioni
\IEEEcompsocitemizethanks{\IEEEcompsocthanksitem T.-P.~Wu  is with the Enterprise and Consumer Electronics Group,
Hong Kong Applied Science and Technology Research Institute (ASTRI), Hong Kong. E-mail: tpwu@astri.org.
\IEEEcompsocthanksitem S.-K.~Yeung and C.-K.~Tang are with the Department of Computer Science and Engineering, Hong Kong University of Science and Technology, Clear Water Bay, Hong Kong.
E-mail: \{saikit,cktang\}@cse.ust.hk.
\IEEEcompsocthanksitem J. Jia is with the Department of Computer Science and Engineering, Chinese University of Hong Kong, Shatin, Hong Kong.
E-mail: leojia@cse.cuhk.edu.hk 
\IEEEcompsocthanksitem G.~Medioni is with the Department of Computer Science,
University of Southern California, USA. E-mail: medioni@usc.edu.
%\protect\\
%note need leading \protect in front of \\ to get a newline within \thanks as
}%
%%%\IEEEcompsocthanksitem J. Doe and J. Doe are with Anonymous University.}% <-this % stops a space
%\thanks{}
}


\markboth{Submitted to IEEE Transactions on Pattern Analysis and Machine Intelligence}%
{Wu \MakeLowercase{\textit{et al.}}: Tensor Voting Revisited: A Closed-Form Solution for Robust Parameter Estimation via Expectation-Maximization}

%\IEEEpubid{\makebox[\columnwidth]{\hfill 0000--0000/00/\$00.00~\copyright~2007 IEEE}%


% PRIOR to the title within the \IEEEcompsoctitleabstractindextext IEEEtran
\IEEEcompsoctitleabstractindextext{%
\begin{abstract}
%\boldmath

We prove a closed-form solution to tensor voting ({\bf CFTV}):
given a point set in any dimensions, our closed-form solution provides
an exact, continuous and efficient algorithm for computing a
structure-aware tensor that simultaneously achieves salient structure
detection and outlier attenuation. Using CFTV, we prove the
convergence of tensor voting on a Markov random field (MRF), thus
termed as {\bf MRFTV}, where the structure-aware tensor at each 
input site reaches a stationary state upon convergence in 
structure propagation.
%We then apply CFTV in two related contributions.
We then embed structure-aware tensor into expectation maximization (EM)
for optimizing a single linear structure to achieve efficient and robust
parameter estimation. Specifically, our {\bf EMTV} algorithm optimizes
both the tensor and fitting parameters and does
not require random sampling consensus typically used in existing
robust statistical
techniques.  We performed quantitative evaluation on its accuracy and robustness,
showing that EMTV performs better than the original TV
and other state-of-the-art techniques
in fundamental matrix estimation for multiview stereo matching.
The extensions of CFTV and EMTV for extracting multiple and nonlinear structures are underway.
An addendum is included in this arXiv version.

%Second, we demonstrate how EMTV benefits multiview
%stereo reconstruction beyond matching.
%Our tensor-based multiview stereo ({\bf TMVS}) combines
%the complementary advantages of photoconsistency, visibility and geometric
%consistency enforcement in multiview stereo via the use of 3D structure-aware
%tensors, where CFTV provides a unified means to manipulate geometric information in
%the match-propagate-filter multiview stereo pipeline.

\end{abstract}
% not want either math or citations to appear in the abstract.

}

\maketitle

% is not selected <OR> if conference mode is selected - because compsoc
% conference papers position the abstract like regular (non-compsoc)
\IEEEdisplaynotcompsoctitleabstractindextext

% \fi
\IEEEpeerreviewmaketitle

\section{Introduction}
\label{sec:intro}
\IEEEPARstart{T}{his} paper reinvents tensor voting~\cite{book_with_names}
for robust computer vision, by proving a closed-form solution to computing 
an {\em exact
structure-aware tensor} after {\em data communication} in a feature space
of any dimensions, where the goal is salient structure inference from 
noisy and corrupted data.

To infer structures from noisy data corrupted by outliers, in tensor
voting, input points communicate among themselves subject to proximity and
continuity constraints. Consequently, each point is
aware of its structure saliency via a structure-aware tensor.
Structure refers to surfaces,
curves, or junctions if the feature space is three dimensional where
a structure-aware tensor can be visualized as an ellipsoid: if
a point belongs to a smooth surface, the resulting ellipsoid
after data communication resembles a stick pointing along the surface normal;
if a point lies on a curve the tensor resembles a plate,
where the curve tangent is perpendicular to the plate tensor;
if it is a point junction where surfaces intersect, the tensor will be like
a ball. An outlier is characterized by a set of inconsistent votes it receives
after data communication.

We develop in this paper a closed-form solution to tensor voting (CFTV),
which is applicable to the special as well as general theory of tensor
voting.  This paper focuses on the {\em special} theory,  where
the above data communication is data driven without using constraints
other than proximity and continuity. The special theory, sometimes
coined as ``first voting pass,'' is applied to process raw input data
to detect structures and outliers. In addition to structure detection
and outlier attenuation, in the {\em general} theory of tensor voting,
tensor votes are propagated along preferred directions to achieve data
communication when such directions are available, typically after
the first pass, such that useful tensor votes are reinforced whereas
irrelevant ones are suppressed.

\begin{figure*}[t]
\begin{center}
\includegraphics[width=0.99\linewidth]{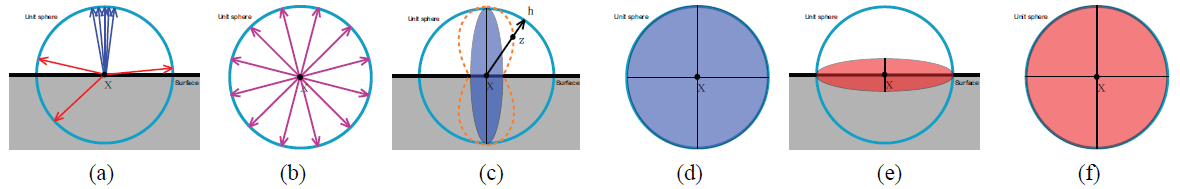}
\end{center}
\caption{\small
Inlier/outlier and tensor inverse illustration.
(a) The normal votes received
at a surface point cast by points in $\bx$'s neighborhood.
Three salient outliers are present.
(b)
For a non-surface point,
there is no preference to any normals.
(c) The structure-aware tensor induced by the normal observations in (a),
which is represented by a $d$-D ellipsoid, where $d \ge 2$.
The orange curve (dashed curve) represents the variance produced
along all possible directions.
(d) The structure-aware tensor collected after collecting the 
received votes in (b).
(e) and (f) correspond to the inverse of (c) and (d), respectively.
}
\label{fig:inlier_outlier_illustration}
\vspace{-0.2in}
\end{figure*}

Expressing tensor voting in a single and compact equation,
or a closed-form solution, offers many advantages:
not only an exact and efficient solution can be achieved with
less implementation effort for salient structure detection
and outlier attenuation, formal and useful mathematical
operations such as differential calculus can be applied
which is otherwise impossible using the original tensor
voting procedure. Notably, we can prove the convergence
of tensor voting on Markov random fields (MRFTV) where a 
structure-aware tensor at each input site achieves a 
stationary state upon convergence. 

Using CFTV, 
%we 
%derive two further contributions, one is theoretical while 
%the other is more application oriented. The first contribution is
we contribute a mathematical derivation based on expectation 
maximization (EM) that
applies the exact tensor solution for extracting the most salient
linear structure, despite that the input data is highly corrupted.
Our algorithm is called EMTV, which optimizes both the tensor and
fitting parameters upon convergence and does not require random
sampling consensus typical of existing robust statistical techniques.
The extension to extract salient multiple and nonlinear structures is
underway.  
%The second contribution consists of an
%unified tensor communication in implementing the match-propagate-filter pipeline
%in multiview stereo. Our work is inspired by the patch-based multiview stereo,
%a state-of-the-art technique in multiview stereo reconstruction~\cite{furukawa_pami09}.
%Our algorithm, called tensor-based multiview stereo (TMVS)~\cite{tmvs}, combines the
%complementary advantages of photoconsistency, visibility and geometric
%consistency enforcement in multiview stereo via the use of 3D structure-aware tensors.

While the mathematical derivation may seem involved, our main
results for CFTV, MRFTV and
EMTV, that is, 
Eqns~(\ref{eqn:cal_S}), (\ref{eqn:cal_S_inv}), 
(\ref{eqn:energy_update_rule}), (\ref{eqn:sor_update_rule}),
(\ref{eqn:e_step_update}), and (\ref{eqn:m_step_update}).
have rigorous mathematical foundations, are applicable
to any dimensions, produce more robust and accurate results as demonstrated
in our qualitative and quantitative evaluation using challenging
synthetic and real data, but on the other hand are easier to implement.
The source codes accompanying this paper are available in the
supplemental material.

\section{Related Work}
\label{sec:related}
While this paper is mainly concerned with tensor voting, we provide
a concise review on robust estimation and expectation maximization.
%Works on multiview stereo will be discussed in section~\ref{sec:mvs}
%where we applied CFTV, MRFTV and EMTV in the match-propagate-filter pipeline.

\noindent {\bf Robust estimators.}
Robust techniques are widely used and an excellent review
of the theoretical foundations of robust
methods in the context of computer vision can be found
in~\cite{meer_bookchapter}.

The Hough transform~\cite{hough_reference} is a robust voting-based
technique operating in a parameter space capable of extracting multiple
models from noisy data.
%However, the complexity
%on storage requirement, which increases exponentially with number
%of model parameters, and the choice of bin size are challenging
%issues.
Statistical Hough transform~\cite{dahyot_pami2009} can be used
for high dimensional spaces with sparse observations.
Mean shift~\cite{meanshift} has been widely used since its introduction
to computer vision for robust feature space analysis. The Adaptive mean
shift~\cite{adaptive_meanshift} with variable bandwidth in high dimensions
was introduced in texture classification and has since been applied to other
vision tasks. Another popular robust method in computer vision
is in the class of random sampling consensus (RANSAC) procedures
\cite{fischler_acm81} which have spawned a lot of follow-up work
(e.g., optimal randomized RANSAC~\cite{chum_pami08}).

Like RANSAC~\cite{fischler_acm81}, robust estimators
including the LMedS~\cite{rousseeuw84} and the M-estimator~\cite{huber81}
adopted a statistical approach. The
LMedS, RANSAC and the Hough transform can be expressed
as M-estimators with auxiliary scale~\cite{meer_bookchapter}.
The choice of scales
and parameters related to the noise level are major issues.
Existing works on robust scale estimation use random
sampling~\cite{subbarao_meer06}, or operate on
different assumptions (e.g., more than 50\% of the data should be
inliers~\cite{rousseeuw_1987}; inliers have a Gaussian
distribution~\cite{lee_1998}). Among them, the Adaptive Scale Sample
Consensus (ASSC) estimator~\cite{wang_and_suter} has shown the best
performance where the estimation process requires no free parameter as input.
Rather than using a Gaussian distribution to model inliers, the authors of 
\cite{wang_and_suter} proposed to use a two-step scale estimator (TSSE) to refine the model scale:
first, a non-Gaussian distribution is used to model inliers where
local peaks of density are found by mean shift~\cite{meanshift};
second, the scale parameter is estimated by a median scale estimator with
the estimated peaks and valleys. 
%TSSE was integrated into a RANSAC-like
%algorithm which together constitute the ASSC method for outlier rejection.
On the other hand, the projection-based M-estimator (pbM)~\cite{chen_meer_iccv03}, 
an improvement made on the M-estimator, uses a Parzen window for scale
estimation, so the scale parameter is automatically found by
searching for the normal direction (projection direction)
that maximizes the sharpest peak of the density. This
does not require an input scale from the user. While these
recent methods can tolerate more outliers, most of them still
rely on or are based on RANSAC and a number of random
sampling trials is required to achieve the desired robustness.

To reject outliers, a multi-pass method using $L_\infty$-norms was
proposed to successively detect outliers which are characterized by
maximum errors~\cite{sim_hartley_cvpr06}.

\noindent {\bf Expectation Maximization.}
%Our work is inspired by Tensor Voting, and uses Expectation Maximization
%for parameters optimization.
%%\subsection{Expectation Maximization}
%%Assuming an overall pattern of distribution, an outlier is an observation
%%that lies outside of the distribution.
%%{\em Expectation-Maximization} 
EM has been used in handling missing
data and identifying outliers in robust computer
vision~\cite{computer_vision_book},
and its convergence properties were studied~\cite{embook}.
In essence, EM consists of two steps~\cite{bilmes97gentle,embook}:
\begin{enumerate}
\item {\bf E-Step}. Computing an expected value for the complete data set using 
incomplete data and the current estimates of the parameters.
\item {\bf M-Step}. Maximizing the complete data log-likelihood using the 
expected value computed in the E-step.
\end{enumerate}

EM is a powerful inference algorithm, but it is also well-known from 
\cite{computer_vision_book} that:
1) initialization is an issue because EM can get stuck in poor local minima,
and 2) treatment of data points with small expected weights requires great
care. They should not be regarded as negligible, as their aggregate effect can
be quite significant.  In this paper we initialize EMTV using structure-aware
tensors obtained by CFTV. As we will demonstrate, such initialization not 
only allows the EMTV algorithm to converge quickly (typically within 20
iterations) but also produces accurate and robust solution in parameter
estimation and outlier rejection.

\section{Data Communication}
\label{sec:smooth_connection}
In the tensor voting framework, a data point, or voter, communicates with
another data point, or vote receiver,  subject to proximity and continuity
constraints, resulting in a tensor vote cast from the voter to the vote
receiver (Fig.~\ref{fig:inlier_outlier_illustration}). In the following, we use $\tn$ to denote a unit voting stick tensor,
${\bf v}$  to denote a stick tensor vote received. Stick tensor vote 
may not be unit vectors when they are multiplied by vote strength.
These stick tensors are building elements of a structure-aware tensor vote.
%As vote strength is sometimes
%multipled to these vectors in mathematical manipulation, we denote
%non-unit vectors by an overhead tilde (e.g. $\ttn$).

Here, we first define a decay function $\eta$ to encode the proximity
and smoothness constraints (Eqns~\ref{eqn:neighbor_weighting} and \ref{eqn:neighbor_weighting_Gaussian}).
While similar in effect to the decay function used in the original
tensor voting, and also to the one used in~\cite{franken_eccv06}
where a vote attenuation function is
defined to decouple proximity and curvature terms, our modified
function, which also differs from that in~\cite{jmlr},
enables a closed-form solution for tensor voting without resorting
to precomputed discrete voting fields.

\begin{figure}[t]
\begin{center}
\includegraphics[width=0.6\linewidth]{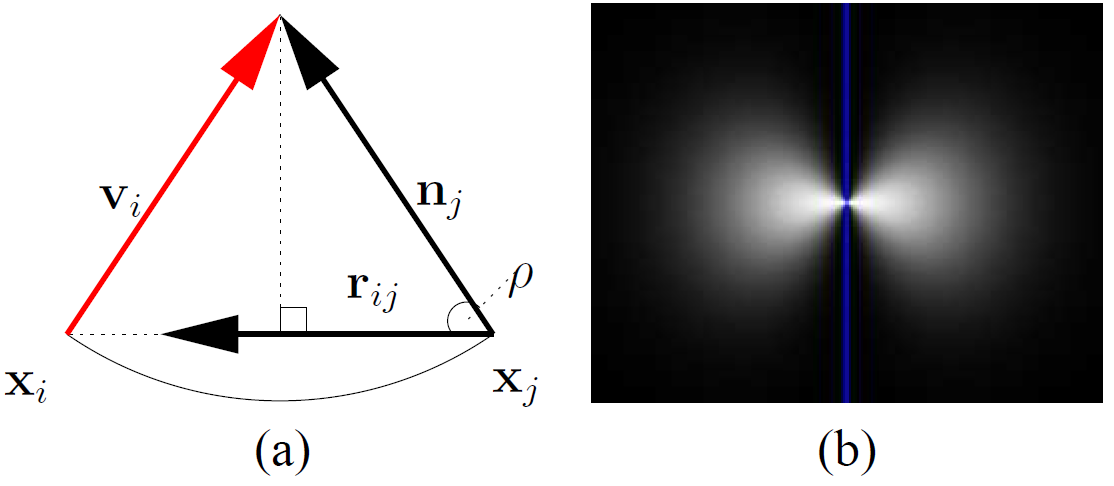}
\end{center}
\caption{\small 
(a)  The normal vote ${\mathbf v}_i$ received at $\bx_i$ using an arc of the
osculating circle between $\bx_i$ and $\bx_j$, assuming the normal voter
at $\bx_j$ is $\tn_j$, where $\tn_j$, ${\mathbf r}_{ij}$, and ${\mathbf v}_i$
are unit vectors in this illustration.
(b) Plot of Eqn.~(\ref{eqn:neighbor_weighting}) in 2D.
}
\label{fig:osculating_arc}
\label{fig:weighted_color_plot}
\vspace{-0.2in}
\end{figure}
Refer to Fig.~\ref{fig:osculating_arc}.
Consider two points $\bx_i \in \Bbb{R}^d$ and $\bx_j \in \Bbb{R}^d$
(where $d > 1$ is the dimension) that are connected by some smooth
structure in the feature/solution space.
Suppose that the unit normal $\tn_j$ at $\bx_j$ is known.
We want to generate at $\bx_i$ a normal (vote) ${\mathbf v}_i$
so that we can calculate $\mathbf{K}_i \in \Bbb{R}^d \times \Bbb{R}^d$, where $\mathbf{K}_i$ is
the structure-aware tensor at $\bx_i$ in the presence of
 ${\mathbf n}_j$ at $\bx_j$. In tensor voting, a 
structure-aware tensor is a second-order symmetric tensor 
which can be visualized as an ellipsoid.

While many possibilities exist, the unit
direction ${\mathbf v}_i$ can be derived by fitting an arc of the
osculating circle between the two points.  Such an arc keeps
the curvature constant along the hypothesized connection, thus
encoding the smoothness constraint.
$\mathbf{K}_i$ is then given by ${\mathbf v}_i {\mathbf v}_i^T$ multiplied by
$\eta(\bx_i, \bx_j, \tn_j)$ defined as:
\begin{equation}
\label{eqn:neighbor_weighting}
\eta( \bx_i, \bx_j, \tn_j ) = c_{ij} ( 1 - ( \mathbf{r}_{ij}^T \tn_j )^2  )
\end{equation}
where
\begin{equation}
\label{eqn:neighbor_weighting_Gaussian}
c_{ij} = \exp( - \frac{||\bx_i - \bx_j ||^2 }{ \sigma_d })
\end{equation}
is an exponential function using Euclidean distance for attenuating 
the strength based on proximity.
$\sigma_d$ is the size of local neighborhood (or the scale parameter,
the only free parameter in tensor voting).

In Eqn.~(\ref{eqn:neighbor_weighting}), 
${\mathbf r}_{ij} \in \Bbb{R}^d$ is a 
unit vector at $\bx_j$ pointing to $\bx_i$, and
$1 - ( {\mathbf r}_{ij}^T \tn_j )^2$ is a squared-sine
function\footnote{
$\sin^2 \rho = 1 - \cos^2 \rho$, where $\cos^2 \rho= ({\mathbf
r}_{ij}^T \tn_j)^2$ and $\rho$ is the angle between
${\mathbf r }_{ij}$ and $\tn_j$.} for attenuating the
contribution according to curvature. Similar to the original
tensor voting framework, Eqn.~(\ref{eqn:neighbor_weighting}) favors
nearby neighbors that produce small-curvature connections,
thus encoding the smoothness constraint. A plot of the 2D version of
Eqn.~(\ref{eqn:neighbor_weighting}) is shown in
Fig.~\ref{fig:weighted_color_plot}(b), where $\bx_j$ is located at
the center of the image and ${\mathbf n}_j$ is aligned with the blue
line. The higher the intensity, the higher the value
Eqn.~(\ref{eqn:neighbor_weighting}) produces at a given pixel location.

Next, consider the general case where the normal $\mathbf{n}_j$ at 
$\bx_j$ is unavailable. Here, let $\bK_j$ at $\bx_j$ be any 
second-order symmetric tensor, which is typically initialized as 
an identity matrix if no normal information is available. 

To compute $\mathbf{K}_i$ given $\bK_j$, we consider equivalently 
the set of all possible unit normals $\{ \tn_{\theta j} \}$ associated with 
the corresponding length $\{ \tau_{\theta j} \}$ which make up 
$\bK_j$ at $\bx_j$, where  $\{ \tn_{\theta j} \}$ and $\{ \tau_{\theta j} \}$ 
are respectively indexed by all possible directions $\theta$. 
Each $\tau_{\theta j} \tn_{\theta j}$ postulates a normal vote 
${\bf v}_{\theta}(\bx_i,\bx_j)$ at $\bx_i$ under the same
smoothness constraint prescribed by the corresponding arc of the 
osculating circle as illustrated in Fig.~\ref{fig:general}.

Let $\bS_{ij}$ be the second-order symmetric {\bf tensor vote} obtained
at $\bx_i$ due to this complete set of normals at $\bx_j$ defined above. 
We have
\begin{equation}
\label{eqn:S_integration}
\bS_{ij} = \int_{ \mathbf{N}_{\theta j} \in \nu } {\bf v}_{\theta}(\bx_i, \bx_j )
{\bf v}_{\theta}(\bx_i, \bx_j)^T \eta(\bx_i, \bx_j, \tn_{\theta j}) d\mathbf{N}_{\theta j}
\end{equation}
where
\begin{equation}
\mathbf{N}_{\theta j} = \tn_{\theta j} \tn_{\theta j}^T
\end{equation}
and $\nu$ is the space containing all possible $\mathbf{N}_{\theta
j}$. For example, if $\nu$ is 2D, the complete set of unit normals
$\tn_{\theta}$ describes a unit circle. If $\nu$ is 3D, the complete
set of unit normals $\tn_{\theta}$ describes a unit sphere\footnote{
The domain of integration $\nu$ represents the space of stick tensors given 
by $\tn_{\theta j}$. Note that $d > 1$;  alternatively, it can be 
understood by expressing $\mathbf{N}_{\theta j}$ using polar coordinates; 
and thus in $N$ dimensions, $\theta = (\phi_1, \phi_2, \cdots, \phi_{n-1})$. 
%Thus, when $\nu$ is concerned, the dimensionality must be greater than 1.
It follows naturally we do not use $\theta$  to define the integration
domain, because rather than simply writing

$ \int_{ \mathbf{N}_{\theta j} \in \nu } \cdots d  \mathbf{N}_{\theta j} $,
   it would have been

$  \int_{\phi_1} \int_{\phi_2}  \cdots \int_{\phi_{n-1}}  \cdots  d\phi_{n-1} d\phi_{n-2} \cdots d\phi_1 $
making the derivation of the proof of Theorem 1 more complicated.
}.

\begin{figure}[t]
\begin{center}
\includegraphics[width=0.3\linewidth]{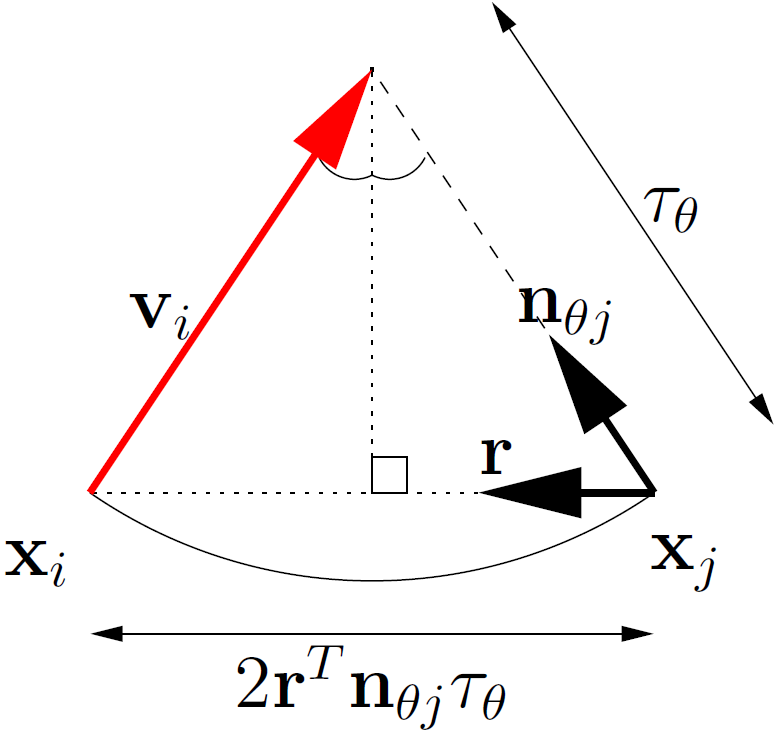} 
\end{center}
\caption{\small
Illustration of Eqn. (\ref{eqn:v}). Normal vote ${\mathbf v}_i = {\mathbf v}_\theta( \bx_i, \bx_j)$ 
received at $\bx_i$ using the arc of the osculating circle between $\bx_i$ and $\bx_j$, 
considering one of the normal voters at $\bx_j$ is $\tn_{\theta j}$. Here,
$\tn_{\theta j}$ and ${\mathbf r}$ are unit vectors.
}
\label{fig:general}
\vspace{-0.2in}
\end{figure}

In a typical tensor voting implementation, Eqn.~(\ref{eqn:S_integration}) is
precomputed as discrete voting fields
(e.g., plate and ball voting fields in 3D tensor voting~\cite{book_with_names}):
the integration is implemented by rotating and summing the
contributions using matrix addition.
Although precomputed once, such discrete approximations
involve uniform and dense sampling of tensor votes
$\tn_{\theta} \tn_{\theta}^T$ in higher dimensions where
the number of dimensions depends on the problem.
In the following section, we will prove a closed-form solution to
Eqn.~(\ref{eqn:S_integration}), which provides an efficient and
exact solution to computing $\bK$ without resorting to discrete
and dense sampling.

%At first glance, it is tedious to implement Eqn.~\ref{eqn:S_integration} because
%uniform sampling is an issue in high dimensions. Moreover, time complexity
%is another issue because dense sampling is required for high precision.
%Fortunately, there exists a closed-form solution to Eqn.~\ref{eqn:S_integration}.

\begin{figure*}[t]
\begin{center}
\includegraphics[width=0.99\linewidth]{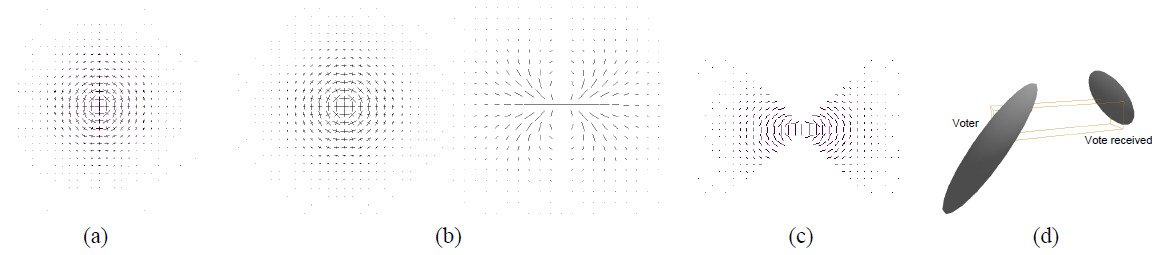} 
\end{center}
\vspace{-0.1in}
\caption{\small (a) {\bf 3D ball voting field}. A slice generated using the closed-form solution Eqn~(\ref{eqn:cal_S}), which has similar tensor orientations (but different tensor strengths) as the ball voting field in~\cite{book_with_names}. 
(b) {\bf 3D plate voting field}. Left: a cut of the voting 
field (direction of $\hat{e}_3$
normal to the page). Right: a cut of the same voting field,
showing the $(\lambda_2 - \lambda_3) \hat{e}_3$ component
(i.e., component parallel to the tangent direction). The field is generated by
using Eqn~(\ref{eqn:cal_S}) showing similar tensor orientations as the
plate voting field in~\cite{book_with_names}. 
(c) {\bf 3D stick voting field}. A slice after zeroing out votes lying in
the 45-degree zone as done in~\cite{book_with_names}. The stick tensor
orientations shown in the figure are identical to those in the 3D stick voting field in
\cite{book_with_names}.
(d) Vote computation using
the closed-form solution in one single step by Eqn~(\ref{eqn:cal_S}).
%%No decomposition into elementary tensors is needed as in~\cite{jmlr}.
}
\label{fig:illus}
\vspace{-0.2in}
\end{figure*}

\section{Closed-Form Solution}

\begin{theorem}
\emph{(Closed-Form Solution to Tensor Voting)}
\label{thm:cftv}
The tensor vote at $\bx_i$ induced by $\bK_j$ located at $\bx_j$ is given by
the following closed-form solution:
\begin{equation*}
\bS_{ij} = c_{ij} \bR_{ij} \bK_j \bR'_{ij}
\end{equation*}
where $\bK_j$ is a second-order symmetric tensor,
$\bR_{ij} ={\mathbf I} - 2 {\mathbf r}_{ij} {\mathbf r}_{ij}^T$,
$\bR'_{ij} = ({\mathbf I} - \frac{1}{2} {\mathbf r}_{ij} {\mathbf
r}_{ij}^T) \bR_{ij}$, 
${\mathbf I}$ is an identity,
${\mathbf r}_{ij}$ is a unit vector pointing from
$\bx_j$ to $\bx_i$ and $c_{ij} = \exp( - \frac{||\bx_i - \bx_j ||^2 }{ \sigma_d })$
with $\sigma_d$ as the scale parameter.
\end{theorem}

\begin{proof}
%Without loss of generality, we consider $\bx_i$ and $\bx_j$ only.
For simplicity of notation, set ${\mathbf r} = {\mathbf
r}_{ij}$, $\tn_{\theta} = \tn_{\theta j}$ and ${\mathbf N}_{\theta}
= {\mathbf N}_{\theta j}$. Now, using the above-mentioned
osculating arc connection,
${\bf v}_{\theta}(\bx_i, \bx_j)$ can be expressed as
\begin{equation}
{\bf v}_{\theta}(\bx_i, \bx_j) = ( \tn_{\theta} - 2 \mathbf{r}
(\mathbf{r}^T \tn_{\theta}) ) \tau_{\theta}
\label{eqn:v}
\end{equation}
Recall that $\tn_{\theta}$ is the unit normal at $\bx_j$ with
direction $\theta$, and that $\tau_{\theta}$ is the length
associated with the normal.  This vector subtraction equation is 
shown in Fig.~\ref{fig:general} where the roles of 
${\bf v}_{\theta}, \tn_{\theta},  \mathbf{r}$, and $\tau_{\theta}$ 
are illustrated.

Let
\begin{equation}
\label{eqn:closed_derived_start} \mathbf{R} = ( \mathbf{I} - 2
\mathbf{r} \mathbf{r}^T ),
\end{equation}
where $\mathbf{I}$ is an identity, we can rewrite
Eqn.~(\ref{eqn:S_integration}) into the following
form:
\begin{equation}
\label{eqn:closed_derive_start} \bS_{ij} = c_{ij}
\int_{\mathbf{N}_{\theta} \in \nu} \tau_{\theta}^2 \mathbf{R}
\tn_{\theta} \tn_{\theta}^T \mathbf{R}^T ( 1 - (\tn_{\theta}^T
\mathbf{r})^2  ) d\mathbf{N}_{\theta} .
\end{equation}

Following the derivation:
\begin{eqnarray}
\nonumber \bS_{ij} &=& c_{ij} \int_{\mathbf{N}_{\theta} \in \nu} \tau_{\theta}^2 \mathbf{R} \tn_{\theta} ( 1 - (\tn_{\theta}^T \mathbf{r})^2  )  \tn_{\theta}^T \mathbf{R}^T d\mathbf{N}_{\theta} \\
\nonumber  &=& c_{ij} \mathbf{R} \left( \int_{\mathbf{N}_{\theta} \in \nu} \tau_{\theta}^2 \tn_{\theta} ( 1 - \tn_{\theta}^T \mathbf{r} \mathbf{r}^T \tn_{\theta} )  \tn_{\theta}^T d\mathbf{N}_{\theta} \right) \mathbf{R}^T\\
\nonumber &=& c_{ij} \mathbf{R} \left( \int_{\mathbf{N}_{\theta} \in \nu} \tau_{\theta}^2 \mathbf{N}_{\theta} - \tau_{\theta}^2 \mathbf{N}_{\theta} \mathbf{r} \mathbf{r}^T \mathbf{N}_{\theta} d\mathbf{N}_{\theta} \right) \mathbf{R}^T \\
\label{eqn:cf_inter} &=& c_{ij} \mathbf{R} \left(  \mathbf{K}_j -
\int_{\mathbf{N}_{\theta} \in \nu} \tau_{\theta}^2
\mathbf{N}_{\theta} \mathbf{r} \mathbf{r}^T \mathbf{N}_{\theta}
d\mathbf{N}_{\theta} \right) \mathbf{R}^T
\end{eqnarray}
The integration can be solved by integration by parts. Let
$f(\theta) = \tau_{\theta}^2 \mathbf{N}_{\theta}$, $f'(\theta) =
\tau_{\theta}^2 \mathbf{I}$, $g(\theta) = \frac{1}{2} \mathbf{r}
\mathbf{r}^T \mathbf{N}_{\theta}^2 $ and $g'(\theta) = \mathbf{r}
\mathbf{r}^T \mathbf{N}_{\theta} $, and note that $
\mathbf{N}_{\theta}^q = \mathbf{N}_{\theta}$ for all $q \in
\Bbb{Z}^+$ 
(see this footnote\footnote{The derivation is as follows $\mathbf{N}_{\theta}^q = \tn_{\theta}  \tn_{\theta}^T  \tn_{\theta}  \tn_{\theta}^T \cdots  \tn_{\theta}  \tn_{\theta}^T =  \tn_{\theta} \cdot 1 \cdot 1 \cdots 1 \cdot  \tn_{\theta}^T =  \tn_{\theta}  \tn_{\theta}^T = \mathbf{N}_{\theta}$.}), 
and ${\mathbf K}_j$, 
in the most general form, can be
expressed as a generic tensor $\int_{\mathbf{N}_{\theta} \in \nu}
\tau_{\theta}^2 \mathbf{N}_{\theta} d\mathbf{N}_{\theta}$. So we
have
\begin{eqnarray}
\nonumber & &\int_{\mathbf{N}_{\theta} \in \nu} \tau_{\theta}^2 \mathbf{N}_{\theta} \mathbf{r} \mathbf{r}^T \mathbf{N}_{\theta} d\mathbf{N}_{\theta} \\
\nonumber &=& \left[ f(\theta) g(\theta) \right]_{\mathbf{N}_{\theta} \in \nu} - \int_{\mathbf{N}_{\theta} \in \nu} f'(\theta) g(\theta) d\mathbf{N}_{\theta} \\
\nonumber &=& \left[ \frac{1}{2} \tau_{\theta}^2 \mathbf{N}_{\theta}
\mathbf{r} \mathbf{r}^T \mathbf{N}_{\theta}^2
\right]_{\mathbf{N}_{\theta} \in \nu}
- \frac{1}{2} \int_{\mathbf{N}_{\theta} \in \nu} \tau_{\theta}^2 \mathbf{r} \mathbf{r}^T \mathbf{N}_{\theta} d\mathbf{N}_{\theta} \\
%\end{eqnarray}
%Here, we rewrite the first term by the product rule for derivative
%and the fundamental theorem of calculus and then express part of the
%second term by a generic tensor. We obtain:
%{\small
%\begin{eqnarray}
%\nonumber & & \frac{1}{2} \int_{{\mathbf N}_{\theta} \in \nu}
%\frac{d}{d{\mathbf N}_{\theta}}[ \tau_{\theta}^2 {\mathbf N}_{\theta}
%{\mathbf r} {\mathbf r}^T {\mathbf N}_{\theta}^2 ] d{\mathbf N}_{\theta} - \frac{1}{2} {\mathbf r} {\mathbf r}^T {\mathbf K}_j \\
\nonumber &=& \frac{1}{2} \int_{{\mathbf N}_{\theta} \in \nu} \left( \tau_{\theta}^2 \frac{d}{d{\mathbf N}_{\theta}}[{\mathbf N}_{\theta}] {\mathbf r} {\mathbf r}^T {\mathbf N}_{\theta}^2
+ \tau_{\theta}^2 {\mathbf N}_{\theta} \frac{d}{d{\mathbf N}_{\theta}}[{\mathbf r} {\mathbf r}^T {\mathbf N}_{\theta}^2] \right) d{\mathbf N}_{\theta} \\
\nonumber & & - \frac{1}{2} {\mathbf r} {\mathbf r}^T {\mathbf K}_j.
\end{eqnarray}
%}
The explanation for the last equality above is given in this footnote\footnote{
Here, we rewrite the first term by the product rule for derivative
and the fundamental theorem of calculus and then express part of the
second term by a generic tensor. We obtain:
\[
\frac{1}{2} \int_{{\mathbf N}_{\theta} \in \nu}
\frac{d}{d{\mathbf N}_{\theta}}[ \tau_{\theta}^2 {\mathbf N}_{\theta}
{\mathbf r} {\mathbf r}^T {\mathbf N}_{\theta}^2 ] d{\mathbf N}_{\theta} - \frac{1}{2} {\mathbf r} {\mathbf r}^T {\mathbf K}_j.
\]
}.

Finally, we apply the fact that $\mathbf{N}_{\theta}^q = \mathbf{N}_{\theta}$ (for all $q \in
\Bbb{Z}^+$) to convert $\frac{d}{d{\mathbf N}_{\theta}}[{\mathbf r} {\mathbf r}^T {\mathbf N}_{\theta}^2]$
into $\frac{d}{d{\mathbf N}_{\theta}}[{\mathbf r} {\mathbf r}^T {\mathbf N}_{\theta}]$. We obtain
\begin{eqnarray}
\nonumber & & \frac{1}{2} \int_{{\mathbf N}_{\theta} \in \nu}
\left( \tau_{\theta}^2 {\mathbf r} {\mathbf r}^T {\mathbf N}_{\theta}^2 +
\tau_{\theta}^2 {\mathbf N}_{\theta}
{\mathbf r} {\mathbf r}^T \right) d{\mathbf N}_{\theta} - \frac{1}{2} {\mathbf r} {\mathbf r}^T {\mathbf K}_j \\
\nonumber &=& \frac{1}{2} \left( {\mathbf r} {\mathbf r}^T {\mathbf K}_j + {\mathbf K}_j {\mathbf r} {\mathbf r}^T - {\mathbf r} {\mathbf r}^T {\mathbf K}_j  \right) \\
\label{eqn:int_by_part_result} &=& \frac{1}{2} {\mathbf K}_j \mathbf{r} \mathbf{r}^T.
\end{eqnarray}
By substituting Eqn. (\ref{eqn:int_by_part_result}) back to Eqn. (\ref{eqn:cf_inter}),
we obtain the result as follows:
\begin{equation}
\bS_{ij}= c_{ij} \bR  \bK_j \left(  \mathbf{I} - \frac{1}{2}  \mathbf{r} \mathbf{r}^T \right) \bR^T. \\
\end{equation}
Replace $\mathbf{r}$ by ${\mathbf r}_{ij}$ such that $\bR_{ij} =
{\mathbf I} - 2 {\mathbf r}_{ij} {\mathbf r}_{ij}^T$ and let
$\bR'_{ij} = ({\mathbf I} - \frac{1}{2} {\mathbf r}_{ij} {\mathbf
r}_{ij}^T) \bR_{ij}$, we obtain
\begin{equation}
\bS_{ij} = c_{ij} \bR_{ij} \bK_j \bR'_{ij}.
\label{eqn:cal_S}
\end{equation}
\end{proof}

A {\em structure-aware tensor} $\bK_i = \sum_j \bS_{ij}$ can thus be assigned at
each site $\bx_i$. This tensor sum considers both geometric proximity
and smoothness constraints in the presence of neighbors $\bx_j$ under
the chosen scale of analysis. Note also that Eqn. (\ref{eqn:cal_S}) is an
exact equivalent of Eqn. (\ref{eqn:S_integration}),
or (\ref{eqn:closed_derive_start}), that is, the first principle.
Since the first principle produces a positive semi-definite matrix, Eqn.
(\ref{eqn:cal_S}) still produces a positive semi-definite matrix. 

In tensor voting,
eigen-decomposition is applied to a structure-aware tensor. In
three dimensions, the eigensystem has eigenvalues
$\lambda_1 \ge \lambda_2 \ge \lambda_3 \ge 0$ with the corresponding
eigenvectors $\hat{e}_1, \hat{e}_2$, and $\hat{e}_3$. $\lambda_1 - \lambda_2$
denotes surface saliency with normal direction indicated by $\hat{e}_1$;
$\lambda_2 - \lambda_3$ denotes curve saliency with tangent direction
indicated by $\hat{e}_3$; junction saliency is indicated by
$\lambda_3$.

While it may be difficult to observe any geometric intuition directly
from this closed-form solution,  the geometric meaning of the
closed-form solution has been described by Eqn. (\ref{eqn:S_integration})
(or (\ref{eqn:closed_derive_start}), the first principle),
since Eqn. (\ref{eqn:cal_S}) is equivalent to Eqn.~(\ref{eqn:S_integration}). Note that our solution is different from, for instance,~\cite{jmlr},
where the $N$-D formulation is approached from a more geometric point of view.
%We, however, are more concerned with a mathematical closed-form solution
%so a rigorous derivation style is adopted.

\begin{figure}[t]
\begin{center}
\includegraphics[width=0.55\linewidth]{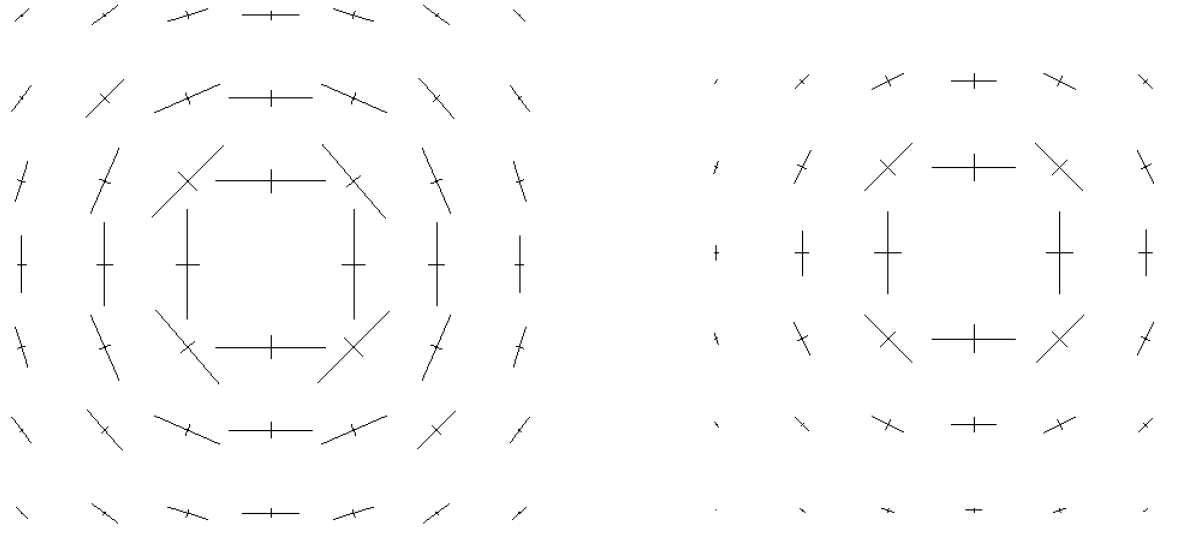}
\end{center}
\caption{\small Close-ups of ball voting fields generated using the original tensor voting
framework (left) and CFTV (right).}
\label{fig:ball_cmp}
\vspace{-0.2in}
\end{figure}

As will be shown in the next section on EMTV, the inverse of $\bK_{j}$ is used.
In case of a perfect stick tensor, which can be equivalently represented
as a rank-1 matrix, does not have an inverse. Similar in spirit where a
Gaussian function can be interpreted as an impulse function associated
with a spread representing uncertainty, a similar statistical approach is 
adopted here in characterizing our tensor inverse. Specifically,
the uncertainty is incorporated using a ball tensor, 
where $\epsilon {\mathbf I}$ is added to $\bK_{j}$, 
$\epsilon$ is a small positive constant (0.001) and ${\mathbf I}$ an identity matrix.
Fig.~\ref{fig:inlier_outlier_illustration} shows a tensor 
and its inverse for some selected cases.
The following corollary regarding the inverse of $\bS_{ij}$ is useful:
\begin{corollary}
Let $\bR{''}_{ij} = \bR_{ij} ({\mathbf I} +
{\mathbf r}_{ij} {\mathbf r}_{ij}^T)$ and also note that ${\mathbf
R}_{ij}^{-1} = {\mathbf R}_{ij}$, the corresponding inverse of $\bS_{ij}$
is:

\begin{equation}
\label{eqn:cal_S_inv} \bS_{ij}' = c_{ij}^{-1}
\bR{''}_{ij} \bK_j^{-1} \bR_{ij}
\end{equation}
\end{corollary}

\begin{proof}
This corollary  can simply be proved by applying inverse
to Eqn (\ref{eqn:cal_S}).
\end{proof}

Note the initial $\bK_j$ can be either derived when input 
direction is available, or simply assigned as an
identity matrix otherwise.

\subsection{Examples}

Using Eqn~(\ref{eqn:cal_S}), given any input $\bK_j$ at site $j$
which is a second-order symmetric tensor, the output
tensor $\bS_{ij}$ can be computed directly. Note that $\bR_{ij}$
is a $d \times d$ matrix of the same dimensionality $d$ as $\bK_j$.
To verify our closed-form solution, we perform the following 
to compare with the voting fields used by the original
tensor voting framework (Fig.~\ref{fig:illus}):

\begin{itemize}
\item[(a)] Set $\bK$ to be an identity (ball tensor) in Eqn~(\ref{eqn:cal_S})
and compute all votes $\bS$ in a neighhorhood. This procedure generates the
ball voting field, Fig.~\ref{fig:illus}(a).
\item[(b)] Set $\bK$ to be a plate tensor
$\left[\begin{array}{ccc}
1 & 0 & 0 \\
0 & 1 & 0 \\
0 & 0 & 0
\end{array}\right]$
in Eqn~(\ref{eqn:cal_S}) and compute all votes $\bS$ in a neighborhood.
This procedure generates the plate voting field, Fig.~\ref{fig:illus}(b).
\item[(c)] Set $\bK$ to be a stick tensor
$\left[\begin{array}{ccc}
1 & 0 & 0 \\
0 & 0 & 0 \\
0 & 0 & 0
\end{array}\right]$
in Eqn~(\ref{eqn:cal_S}) and compute all votes $\bS$ in a neighborhood.
This procedure generates the stick voting field, Fig.~\ref{fig:illus}(c).
\item[(d)] Set $\bK$ to be {\em any} generic second-order tensor in
Eqn~(\ref{eqn:cal_S}) to compute a tensor vote $\bS$ at a given site.
We do not need a voting field, or the somewhat complex procedure described
in~\cite{jmlr}. In one single step using the closed-form solution
Eqn~(\ref{eqn:cal_S}), we obtain $\bS$ as shown in Fig.~\ref{fig:illus}(d).
\end{itemize}

Note that the stick voting field generation is the same as the closed-form 
solution given by the arc of an osculating circle. On the other hand, since the
closed-form solution does not remove votes lying beyond the 45-degree zone 
as done in the original framework, it is useful
to compare the ball voting field generated using the CFTV and the original 
framework. Fig.~\ref{fig:ball_cmp} shows the close-ups of the ball voting fields 
generated using the original framework and CFTV. As anticipated, the tensor 
orientations are almost the same (with the maximum angular deviation 
at $4.531^\circ$), while the tensor strength is different due to the use of 
different decay functions. The new computation results in perfect vote 
orientations which are radial, and the angular discrepancies are due 
to the discrete approximations in the original solution.

While the above illustrates the usage of Eqn~(\ref{eqn:cal_S})
in three dimensions, the equation applies to any dimensions $d$. All
of the $\bS$'s returned by Eqn~(\ref{eqn:cal_S}) are second-order symmetric
tensors and can be decomposed using eigen-decomposition.
The implementation of Eqn~(\ref{eqn:cal_S}) is a matter
of a few lines of C++ code.
%which can be obtained from the authors
%website.

%Our closed-form solution disproves the claim in~\cite{jmlr} that
%``the ... integral (Eqn~(\ref{eqn:S_integration})) has no closed form
%solution.''
Our ``voting without voting fields'' method is uniform to 
any input tensors $\bK_j$ that are second-order symmetric tensor 
in its closed-form expressed by Eqn~(\ref{eqn:cal_S}), 
where formal mathematical operation
can be applied on this compact equation, which is otherwise
difficult on the algorithmic procedure described in previous tensor
voting papers.
Notably, using the closed-form solution,
we are now able to prove mathematically the convergence of tensor voting
in the next section.

%that is, all the input tensors will achieve a stationary state no
%matter how many times tensor voting is applied after convergence.
%This will be exemplified in section~\ref{sec:filtering} when we derive
%a quadratic energy in {\tt mrftv} for implementing the filtering step
%in multiview stereo.

\vspace{-0.1in}

\subsection{Time Complexity}
Akin to the original tensor voting formalism, each site (input or
non-input) communicates with each other on an Markov random field 
(MRF) in a broad sense, where the number of edges depends on the scale
of analysis, parameterized by $\sigma_d$ in Eqn
(\ref{eqn:neighbor_weighting_Gaussian}). In our implementation,
we use an efficient data structure such as ANN tree~\cite{ann} to
access a constant number of neighbors $\bx_j$ of each $\bx_i$.
It should be noted that under a large scale of analysis where
the number of neighbors is sufficiently large, similar number
of neighbors are accessed in ours and the original tensor voting
implementation.

The speed of accessing nearest neighbors can be greatly increased
(polylogarithmic) by using ANN thus making efficient the computation of a
structure-aware tensor.  Note that the running time
for this implementation of closed-from solution is $O(d^3)$, while
the running time for the original tensor voting (TV) is $O( u^{d-1} )$,
where $d$ is the dimension of the space and $u$ is the number
of sampling directions for a given dimension. Because of this,
a typical TV implementation precomputes and stores the dense tensor fields.
For example, when $d = 3$ and $u = 180$ for high accuracy, our
method requires 27 operation units, while  a typical TV implementation
requires 32400 operation units.
Given 1980 points and the same number of neighbors,
the time to compute a structure-aware tensor using our method
is about $0.0001$ second; it takes about $0.1$ second for
a typical TV implementation to output the corresponding tensor.
The measurement was performed on a computer running on a core 
duo 2GHz CPU with 2GB RAM.

Note that the asymptotic running time for the improved TV in~\cite{jmlr} is
$O( d \gamma^2 )$ since it applies Gramm-Schmidt process to perform
component decomposition, where $\gamma$ is the number of linearly independent
set of the tensors.
In most of the cases, $\gamma = d$. So, the running time for our method is
comparable to~\cite{jmlr}. However, their approach does not have a
precise mathematical solution.

\section{MRFTV}

We have proved CFTV for the special theory of tensor voting, or the
``first voting pass'' for structure inference. Conventionally, tensor
voting was done in two passes, where the second pass was used for
structure propagation in the preferred direction after disabling the ball
component in the structure-aware tensor. What happens if more tensor
voting passes are applied? This has never been answered properly.

In this section we provide a convergence proof for tensor voting based 
on CFTV: the structure-aware tensor obtained at each site achieves 
a stationary state upon convergence. Our convergence proof makes 
use of Markov random fields (MRF), thus termed as MRFTV. 

It should be noted that the original tensor voting formulation is also
constructed on an MRF according to the broad definition, since random
variables (that is, the tensors after voting) are defined on the nodes
of an undirected graph in which each node is connected to all neighbors
within a fixed distance.  On the other hand, without CFTV, it was
previously difficult to write down an objective function and to prove
the convergence.  One caveat to note in the following 
is that we do not disable the ball component in each iteration, which will be
addressed in the future in developing the general theory of tensor voting
in structure propagation.
As we will demonstrate, MRFTV does not smooth out important features 
(Fig.~\ref{fig:L}) and still possesses high outlier rejection ability 
(Fig.~\ref{fig:filtering}).

\begin{figure}[t]
\begin{center}
\includegraphics[width=0.98\linewidth]{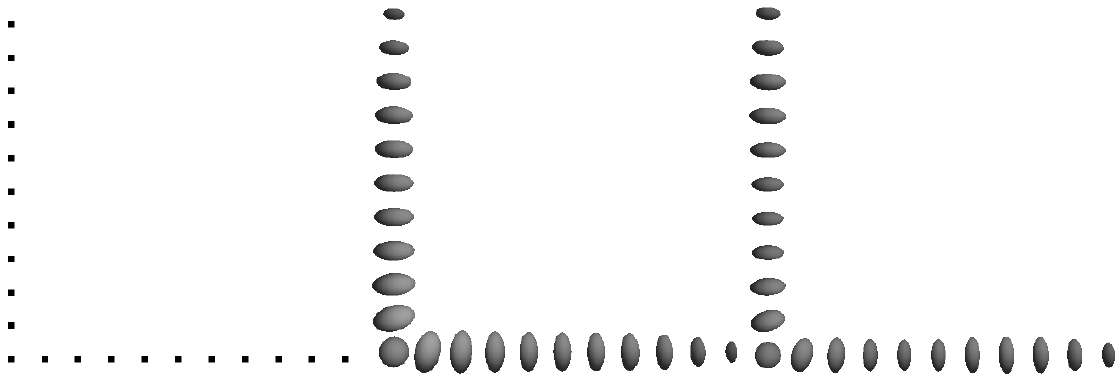} 
\end{center}
\vspace{-0.1in}
\caption{\small Convergence of MRF-TV. From left to right: input points,
result after 2 passes, result after convergence (10 iterations). Visually,
this simple example distinguishes tensor voting from smoothing as the
sharp orientation discontinuity is preserved upon convergence. 
}
\label{fig:L}
\vspace{-0.1in}
\end{figure}

Recall in MRF, a Markov network is a graph consisting of two types of nodes --
a set of hidden variables $\bE$ and a set of observed variables $\bO$, where
the edges of the graph are described by the
following posterior probability $\text{P} (\bE|\bO)$ with
standard Bayesian framework:
\begin{equation}
\text{P} (\bE|\bO) \propto \text{P} (\bO|\bE) \text{P}(\bE)
\label{eqn:bayesian}
\end{equation}

By letting $\bE = \{ \bK_i | i = 1,2,\cdots,N \}$ and
$\bO = \{ \tilde{\bK}_i | i = 1,2, \cdots, N\}$, where $N$ is total number of
points and $\tilde{\bK}_i$ is the known tensor at $\bx_i$, and suppose that
inliers follow Gaussian distribution, we obtain the 
likelihood $\text{P} (\bO|\bE)$ and the prior $\text{P}(\bE)$ as follows:
\begin{eqnarray}
\label{mrf1}
\text{P} (\bO|\bE) &\propto& \prod_i \text{p} (\tilde{\bK}_i | \bK_i) = \prod_i e^{ - \frac{ || \bK_i - \tilde{\bK}_i ||_F^2 }{ \sigma_h} } \\
\text{P}(\bE) &\propto& \prod_i \prod_{j \in {\cal N} (i)} \text{p} (\bS_{ij}|\bK_i ) \\
&=& \prod_i \prod_{j \in {\cal N} (i)} e^{- \frac{ || \bK_i - \bS_{ij} ||_F^2 }{\sigma_s} }
\label{mrf2}
\end{eqnarray}
where $||\cdot ||_F$ is the Frobenius norm,
$\tilde{\bK}_i$ is the known tensor
at $\bx_i$, ${\cal N} (i)$ is the set of neighbor corresponds to $\bx_i$ and
$\sigma_h$ and $\sigma_s$ are two constants respectively. 

Note that we use the Frobenius norm to encode tensor orientation
consistency as well as to reflect the necessary vote saliency 
information including distance and continuity
attenuation. For example, suppose we have a unit stick
tensor at $\bx_i$ and a stick vote (received at $\bx_i$), which
is parallel to it but with magnitude equal to 0.8. In another scenario
$\bx_i$ receives from a voter farther away a stick vote with
the same orientation but magnitude being equal to 0.2.
The Frobenius norm reflects the difference in saliency despite 
the perfect orientation consistency in both cases.  Notwithstanding, it 
is arguable that Frobenius norm may not be the perfect solution
to encode orientation consistency constraint in the pertinent 
equations, while this current form works acceptably well in our 
experiments in practice.

By taking the logarithm of Eqn~(\ref{eqn:bayesian}), we obtain 
the following energy function:
\begin{equation}
E(\bE) = \sum_i || \bK_i - \tilde{\bK}_i ||_F^2 +
g \sum_i \sum_{j \in {\cal N} (i)} || \bK_i - \bS_{ij} ||_F^2
\label{eqn:energy_function}
\end{equation}
where $g = \frac{\sigma_h}{\sigma_s}$. Theoretically, this {\em quadratic}
energy function can be directly solved by Singular Value
Decomposition (SVD).  Since $N$ can be large thus making direct
SVD impractical, we adopt an iterative approach: by taking
the partial derivative of Eqn~(\ref{eqn:energy_function})
(w.r.t. to $\bK_i$) the following update rule is obtained:
\begin{equation}
\label{eqn:energy_update_rule}
\bK_i^* = (\tilde{\bK}_i + 2 g \sum_{j \in {\cal N} (i)} \bS_{ij}) (\bI + g \sum_{j \in {\cal N} (i)} (\bI + c_{ij}^2 {\bR_{ij}'}^2) )^{-1}
\end{equation}
which is a Gauss-Seidel solution.
When successive over-relaxation (SOR) is employed, the update rule becomes:
\begin{equation}
\label{eqn:sor_update_rule}
\bK_i^{(m+1)} = (1 - q) \bK_i^{(m)} + q \bK_i^*
\end{equation}
where $1 < q < 2$ is the SOR weight and $m$ is the iteration number.
After each iteration, we normalize $\bK_i$ such that the eigenvalues 
of the corresponding eigensystem are within the range $(0, 1]$. 

The above proof on convergence of MRF-TV 
shows that structure-aware tensors
achieve stationary states after a finite number Gauss-Seidel 
iterations in the above formulation.
It also dispels a common pitfall that tensor voting is similar in
effect to smoothing. Using the same scale of analysis (that is, in
(\ref{eqn:neighbor_weighting_Gaussian})) and same $\sigma_h$, $\sigma_s$
in each iteration, tensor saliency and orientation will both
converge.  We observe that the converged tensor orientation is in
fact similar to that obtained after two voting passes using the
original framework, where the orientations at curve junctions
are not smoothed out. See Fig.~\ref{fig:L} for an example where sharp
orientation discontinuity is not smoothed out when tensor voting
converges. Here, $\lambda_1$ of each structure-aware tensor is not
normalized to 1 for visualizing its structure saliency after convergence. 
Table~\ref{tab:conv} summarizes the quantitative comparison with the
ground-truth orientation.

\if 0
\begin{table}[t]
\begin{center}
\begin{tabular}{cccc}
\hline
& \multicolumn{3}{c}{Tensor orientation angular error (in deg)} \\ 
& min & max & avg \\ \hline
2 passes & 0.000000 & 0.115430 & 0.027264 \\
converged & 0.000000 & 0.045051 & 0.010683 \\ \hline
\end{tabular}
\end{center}
\caption{\small Comparison with ground truth for the example in Fig.~\ref{fig:L}. }
\label{tab:conv}
\vspace{-0.2in}
\end{table}
\fi

\begin{table}[t]
\begin{center}
\includegraphics[width=0.75\linewidth]{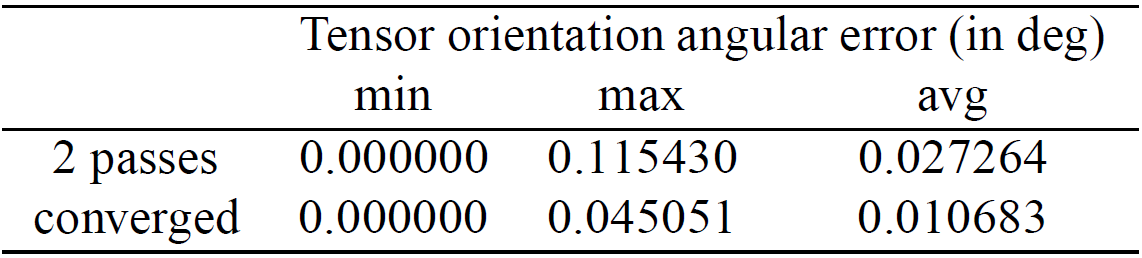}
\end{center}
\vspace{-0.1in}
\caption{\small Comparison with ground truth for the example in Fig.~\ref{fig:L}. }
\label{tab:conv}
\vspace{-0.2in}
\end{table}

\section{EMTV}

%$\{\bK_i^{-1}\} $ will be used in this section for robust plane extraction
%in any dimensions. The extension to non-linear parameter estimation will be
%outlined in the supplementary material.

Previously, while tensor voting was capable of rejecting outliers, it
fell short of producing accurate parameter estimation, explaining the use
of RANSAC in the final parameter estimation step after outlier
rejection~\cite{4dpami}.

This section describes the EMTV algorithm for optimizing
a) the structure-aware tensor $\bK$ at each input site, and b)
the parameters of a single plane $\bv$ of {\em any} dimensionality
containing the inliers.
This algorithm will be applied to stereo matching.
%%to be demonstrated shortly.

%The extensions to multiple planes and
%nonlinear model fitting, outlined in the supplemental material,
%are straightforward extensions based on EMTV and our future
%work is to apply this extension to vision problems.

%Such high dimensional linear model is very general in practice.
%The extension to non-linear parameter estimation will be detailed in
%section~\ref{}.

We first formulate the three constraints to be used in EMTV.
These constraints are not mutually exclusive, where knowing the
values satisfying one constraint will help computing the values of the others. However, in
our case, they are all unknowns, so EM is particularly suitable for
their optimization, since the expectation calculation and parameter
estimation are solved alternately.

\vspace{-0.15in}

\subsection{Constraints}
\label{sec:cases}
%By integrating the following three constraints in the EM algorithm,
%a very robust solution to outlier rejection is resulted.
%Without loss of generality, we assume the inliers, which can be
%in any dimensions, can be characterized by linear structure.
%We will
%outline in the supplementary material how to apply the ``kernel trick''
%to amend our method to non-linear model fitting which can be readily
%done. In essence, kernel tricks are commonly used in machine learning
%research~\cite{} to linearize a non-linear problem by casting the
%problem to a high dimensional space.

\noindent{\bf Data constraint } Suppose we have a set of clean
data. One necessary objective is to minimize
the following for all $\bx_i \in \Bbb{R}^d$ with $d > 1$:
\begin{equation}
\label{eqn:data}
|| \bx_i^T \bv ||
\end{equation}
where $\bv \in \Bbb{R}^d$ is a unit vector representing 
the plane (or the model) to be estimated.\footnote{
Note that, in some cases, the underlying model is represented in this
form $\bx_i^T \bv - z_i$ where we can re-arrange it into
the form given by Eqn.~(\ref{eqn:data}). 
%The easiest way is to adopt similar
%idea in homogeneous coordinates. 
For example, expand $\bx_i^T \bv - z_i$ into $ax_i + by_i + 1 z_i = 0$, 
which can written in the form of (\ref{eqn:data}).
}
%Note that the rearranged $\bx_i$ may not need to be in
%homogeneous coordinates since we do not simply add one extra dimension to $\bx_i$
%and $y_i$ contributes part of the data.}
This is a typical data term that measures the faithfulness of the input data to the fitting plane.

\noindent{\bf Orientation consistency  } The plane being estimated is defined by
the vector $\bv$. Since the tensor $\bK_i \in \Bbb{R}^d \times \Bbb{R}^d$
encodes structure awareness, if
$\bx_i$ is an inlier, the orientation information encoded by $\bK_i$ and $\bv$ have to be
consistent. That is, the variance $\bv^T \bK_i^{-1} \bv$ produced by $\bv$
should be minimal.
Otherwise, $\bx_i$ might be generated by other models even if it minimizes
Eqn~(\ref{eqn:data}).
%\footnote{Our method can be generalized to handle multiple planes (and non-linear models). This is outlined in the supplementary material.
%For presentation clarity and without losing generality, we focus on the
%case of single plane extraction in this paper.}
Mathematically, we minimize:
\begin{equation}
\label{eqn:orientation_consistency}
|| \bv^T \bK_i^{-1} \bv ||.
\end{equation}

%where $ 0 \leq \bv^T \bK_i \bv \leq 1 $ is the variance produced along direction $\bv$ w.r.t. the covariance matrix
%$\bK_i$. If $\bx_i$ is an inlier, the variance produced will be large.

\noindent{\bf Neighborhood consistency  } While the estimated $\bK_i$ helps
to indicate inlier/outlier information, $\bK_i$ has to be consistent with
the local structure imposed by its neighbors (when they are known).
If $\bK_i$ is consistent with $\bv$ but not the local neighborhood,
either $\bv$ or $\bK_i$ is wrong. In practice, we minimize the
following Frobenius norm as in Eqns (\ref{mrf1})--(\ref{mrf2}):
\begin{equation}
\label{eqn:neighborhood_consistency}
|| \bK_i^{-1} - \bS'_{ij} ||_F.
\end{equation}
In the spirit of MRF, $\bS'_{ij}$ encodes the tensor
information within $\bx_i$'s neighborhood, thus a natural choice for
defining the term for measuring neighborhood orientation consistency.
This is also useful, as we will see, to the M-step of EMTV which makes
the MRF assumption.

The above three constraints will interact with each other in
the proposed EM algorithm.

\subsection{Objective Function}
Define $\bO = \{ o_i = \bx_i | i = 1, \cdots , N \}$ to be the set of
observations. Our goal is to optimize $\bv$ and $\bK_i^{-1}$ given
$\bO$. Mathematically, we solve the objective function:
\begin{equation}
\label{eqn:em_objective} \Lambda^* = \arg \max_\Lambda P( \bO, \bR |
\Lambda )
\end{equation}
where $P( \bO, \bR | \Lambda )$ is the complete-data likelihood to
be maximized, $\bR = \{ r_i \}$ is a set of hidden states indicating
if observation $o_i$ is an outlier ($r_i = 0$) or inlier ($r_i = 1$), and
$\Lambda = \{ \{ \bK_i^{-1} \}, \bv, \alpha, \sigma, \sigma_1, \sigma_2 \}$
is a set of parameters to be estimated. $\alpha$, $\sigma$, $\sigma_1$ and
$\sigma_2$ are parameters imposed by some distributions, which will
be explained shortly by using an equation to be introduced\footnote{See the M-step in
Eqn~(\ref{eqn:m_step_update}).}. Our EM algorithm estimates an optimal
$\Lambda^*$ by finding the value of the complete-data log likelihood
with respect to $\bR$ given $\bO$ and the current estimated
parameters $\Lambda'$:
\begin{equation}
\label{eqn:q_function} Q( \Lambda, \Lambda' ) = \sum_{\bR \in \psi}
\log P(\bO, \bR | \Lambda ) P(\bR | \bO, \Lambda')
\end{equation}
where $\psi$ is a space containing all possible configurations of
$\bR$ of size $N$. Although EM does not guarantee a global optimal
solution theoretically, because CFTV provides good initialization,
we will demonstrate empirically that reasonable results can be obtained.
%%using challenging examples.

\vspace{-0.1in}

\subsection{ Expectation (E-Step) }
\label{sec:e_step} In this section, the marginal distribution $p(r_i
| o_i, \Lambda')$ will be defined so that we can maximize the
parameters in the next step (M-Step) given the current parameters.

If $r_i = 1$, the observation $o_i$ is an inlier and therefore
minimizes the first two conditions (Eqns~\ref{eqn:data}
and~\ref{eqn:orientation_consistency}) in
Section~\ref{sec:cases}, that is, the data and orientation
constraints. In both cases, we assume that inliers follow a Gaussian
distribution which explains the use of $\bK_i^{-1}$ instead
of $\bK_i$.\footnote{Although a linear structure is being optimized here, 
the inliers together may describe a structure that does not necessarily 
follow any particular model. Each inlier may not exactly lie on this 
structure where the misalignment follows the Gaussian distribution.}
We model ${\rm p}( o_i | r_i, \Lambda' )$ as
\begin{eqnarray}
\propto \left\{%
\begin{array}{ll}
  \exp(-\frac{|| \bx_i^T \bv ||^2 }{ 2 \sigma^2 } ) \exp(-\frac{|| \bv^T \bK_i^{-1} \bv ||}{ 2 \sigma_1^2 } ),
& \hbox{if $r_i = 1$;} \\
    \frac{1}{C}, & \hbox{if $r_i = 0$.}
\label{eqn:o_i_prob}
\end{array}%
\right.
\end{eqnarray}

%\begin{equation}
%\label{eqn:o_i_prob} p(o_i | r_i = 1, \Lambda') \propto
%\exp(-\frac{|| \bx_i^T \bv ||^2 }{ 2 \sigma^2 } ) \exp(-\frac{|| \bv^T \bK_i^{-1} \bv ||}{ 2 \sigma_1^2 } )
%\end{equation}
%and $p(o_i | r_i = 0, \Lambda') \propto 1 - p(o_i | r_i = 1,
%\Lambda') $. 

We assume that outliers follow uniform distribution where $C$ 
is a constant that models the distribution. Let $C_m$ be the 
maximum dimension of the bounding box of the input. In practice, 
$C_m \leq C \leq 2 C_m$ produces similar results. 

Since we have no prior information on a point being an inlier 
or outlier, we may assume that the mixture probability of the 
observations $p(r_i = 1) = p(r_i = 0)$ equals to a constant 
$\alpha = 0.5$ such that we have no bias to either category 
(inlier/outlier). 
%In that case, $p(o_i | \Lambda' ) = \alpha p(o_i | r_i = 0, \Lambda') +
%\alpha (1 - p(o_i | r_i = 0, \Lambda')) = \alpha$. 
For generality in the following
we will include $\alpha$ in the derivation.

Define $w_i = p(r_i | o_i, \Lambda')$ to be the probability of $o_i$
being an inlier. Then

\if 0
\begin{eqnarray}
\nonumber w_i &= & {\rm p}(r_i =1 | o_i, \Lambda') = 
\frac{ {\rm p}( o_i, r_i = 1 | \Lambda' ) } { {\rm p}( o_i | \Lambda') } \\
& = & \frac{ \alpha ~ \exp( - \frac{|| \bx_i^T \bv ||^2 }{ 2 \sigma^2 } ) \exp( -\frac{|| \bv^T \bK_i^{-1} \bv ||}{ 2 \sigma_1^2 } ) } { \alpha ~ \exp( - \frac{|| \bx_i^T \bv ||^2 }{ 2 \sigma^2 }  ) 
\exp( - \frac{|| \bv^T \bK_i^{-1} \bv || }{ 2 \sigma_1^2 } ) 
+ \frac{ 1 - \alpha } {C} }
\label{eqn:e_step_update} 
\end{eqnarray}
\fi

\begin{eqnarray}
\nonumber w_i &= & {\rm p}(r_i =1 | o_i, \Lambda') =
\frac{ {\rm p}( o_i, r_i = 1 | \Lambda' ) } { {\rm p}( o_i | \Lambda') } \\
& = & \frac{ \alpha \beta ~ \exp( - \frac{|| \bx_i^T \bv ||^2 }{ 2 \sigma^2 } )    \exp( - \frac{|| \bv^T \bK_i^{-1} \bv ||}{ 2 \sigma_1^2 } ) } { \alpha \beta ~   \exp( - \frac{|| \bx_i^T \bv ||^2 }{ 2 \sigma^2 }  )   \exp( - \frac{|| \bv^T \bK_i^{-1} \bv || }{ 2 \sigma_1^2 } )
+ \frac{ 1 - \alpha } {C} }
\label{eqn:e_step_update}
\end{eqnarray}
where $\beta = \frac{1}{2 \sigma \sigma_1 \pi}$ is the normalization term.
%So, in the E-Step, we compute $w_i$ using Eqn.~(\ref{eqn:e_step_update}) for all $i$.

\vspace{-0.1in}

\subsection{ Maximization (M-Step) }
\label{sec:m_step} In the M-Step, we maximize
Eqn.~(\ref{eqn:q_function}) using $w_i$ obtained from the E-Step.
Since neighborhood information is considered, we model $P(\bO, \bR |
\Lambda )$ as a MRF:
\begin{equation}
\label{eqn:mrf} P(\bO, \bR | \Lambda ) = \prod_i \prod_{ j \in {\cal
G}(i)} p(r_i | r_j, \Lambda) p(o_i | r_i, \Lambda)
\end{equation}
where ${\cal G}(i)$ is the set of neighbors of $i$. In theory,
${\cal G}(i)$ contains all the input points except $i$, since
$c_{ij}$ in Eqn.~(\ref{eqn:neighbor_weighting_Gaussian}) is always
non-zero (because of the long tail of the Gaussian distribution). In
practice, we can prune away the points in ${\cal G}(i)$ where the
values of $c_{ij}$ are negligible. This can greatly reduce the size
of the neighborhood. Again, using ANN tree~\cite{ann}, the speed of
searching for nearest neighbors can be greatly increased.

Let us examine the two terms in Eqn.~(\ref{eqn:mrf}).
$p(o_i | r_i, \Lambda)$ has been
defined in Eqn.~(\ref{eqn:o_i_prob}). We define $p(r_i | r_j,
\Lambda)$ here. Using the third condition mentioned in
Eqn.~(\ref{eqn:neighborhood_consistency}), we have:
\begin{equation}
p(r_i | r_j, \Lambda) = \exp( - \frac{ ||\bK_i^{-1} - \bS'_{ij}
||^2_F } {2 \sigma^2_2})
\end{equation}
%Since inliers and outliers follow two different distributions, when $r_i \neq r_j$, $p(r_i| r_j, \Lambda) = 1$.
We are now ready to expand Eqn.~(\ref{eqn:q_function}). Since $r_i$
can only assume two values (0 or 1), we can rewrite $Q(\Lambda,
\Lambda')$ in Eqn.~(\ref{eqn:q_function}) into the following form:

\begin{eqnarray}
%\nonumber
%& & Q(\Lambda, \Lambda') \\
%\nonumber
%&=&
\nonumber \sum_{t \in \{0,1\}} \log ( \prod_i \prod_{ j \in {\cal
G}(i)} p(r_i = t | r_j, \Lambda) p(o_i | r_i = t, \Lambda))
P(\bR | \bO,\Lambda')
%&+&  \sum_{i} \log \left( \prod_i \prod_{ j \in {\cal G}(i)} p(r_i = 0 | r_j, \lambda) p(o_i | r_i = 0) \right) P(\bR %| \bO,\Lambda')
\end{eqnarray}
After expansion,
\begin{eqnarray}
\nonumber
Q(\Lambda , \Lambda') &=& \sum_i \log( \alpha \frac{1}{\sigma \sqrt{2 \pi} } \exp(-\frac{|| \bx_i^T \bv||^2 }{ 2 \sigma^2 } )  ) w_i \\
\nonumber
&+& \sum_i \log( \frac{1}{\sigma_1 \sqrt{2 \pi}} \exp(-\frac{|| \bv^T \bK_i^{-1} \bv ||}{ 2 \sigma_1^2 } ) ) w_i\\    %%% add alpha by CK 07/28/2011  %%% remove alpha by CK 10/25/2011
\label{eqn:q_function_expanded} &+& \sum_i \log( \exp(-\frac{||
\bK_i^{-1} - \bS'_{ij} ||_F^2}{ 2 \sigma_2^2} ) ) w_i w_j \nonumber \\
&+& \sum_i \log( \frac{1 - \alpha}{C} ) ( 1 - w_i )
\end{eqnarray}

To maximize Eqn.~(\ref{eqn:q_function_expanded}), we set the first
derivative of $Q$ with respect to $\bK_i^{-1}$, $\bv$, $\alpha$, $\sigma$,
$\sigma_1$ and $\sigma_2$ to zero respectively to obtain the
following set of update rules:

\begin{eqnarray*}
\nonumber \alpha &=& \frac{1}{N} \sum_i w_i \\
\nonumber \bK_i^{-1} &=&  \frac{1}{\sum_{j \in {\cal G}(i)} w_j} ( \sum_{j \in {\cal G}(i)} \bS'_{ij} w_j - \frac{\sigma_2^2}{ 2 \sigma_1^2 } \bv \bv^T w_i )\\
\nonumber  \min ||  \bM \bv || & & \mbox{subject to $|| \bv || = 1$} \\
%\nonumber \bM \bv &=& 0 \mbox{~~~~(solve for $\bv$)} \\
%\nonumber \bv &=& (\bM_i^T \bM_i)^{-1} \bM_i^T \sum_i \bx_i y_i w_i \\
\nonumber \sigma^2 &=& \frac{\sum_i ||\bx_i^T \bv ||^2 w_i }{ \sum_i w_i} 
\end{eqnarray*}
\begin{eqnarray}
\nonumber \sigma_1^2 &=& \frac{ \sum_i || \bv^T \bK_i^{-1} \bv || w_i }{ \sum_i w_i } \\
\nonumber \label{eqn:m_step_update} \sigma_2^2 &=& \frac{ \sum_i \sum_{j \in {\cal G}(i)} || \bK_i^{-1} - \bS'_{ij} ||_F^2 w_i w_j }{ \sum_i w_i} \\
\end{eqnarray}
where $\bM = \sum_i \bx_i \bx_i^T w_i +
\frac{\sigma^2}{\sigma_1^2} \sum_i \bK_i^{-1} w_i$ and ${\cal G}(i)$
is a set of neighbors of $i$.  
%The third rule can be solved by $\min ||  \bM \bv ||$ subject to $|| \bv || = 1$.
Eqn.~(\ref{eqn:m_step_update}) constitutes the set 
of update rules for the M-step.

%There is no preference for any update rule in the execution sequence for
%Eqn.~(\ref{eqn:m_step_update}), because they can be executed in
%parallel. The newly estimated parameters should be cached for this
%purpose. 

%Kit said: The K, h, they are using the sigma2, sigma1 at the 
%current iteration, but NOT the updated one.
%all the variable are "memcpy" only after all the 5 update rules are
%executed.

In each iteration, after the update rules have been
executed, we normalize $\bK_i^{-1}$ %%%and $\bv$
onto the feasible solution space by normalization, that is,
the eigenvalues of the corresponding eigensystem are
within the range $(0, 1]$.   %%% $[0,\max(\lambda_1)]$.
Also, $\bS'_{ij}$ will be updated
with the newly estimated $\bK_i^{-1}$.

%\begin{figure}
%\begin{center}
%\includegraphics[width=0.4\linewidth]{figure/plots/percent_io_plot.eps}
%\end{center}
%\caption{The plot of $Z = \frac{R}{R+1}$}
%\label{fig:plot_z}
%\end{figure}

\begin{figure*}
\begin{center}
\includegraphics[width=0.95\linewidth]{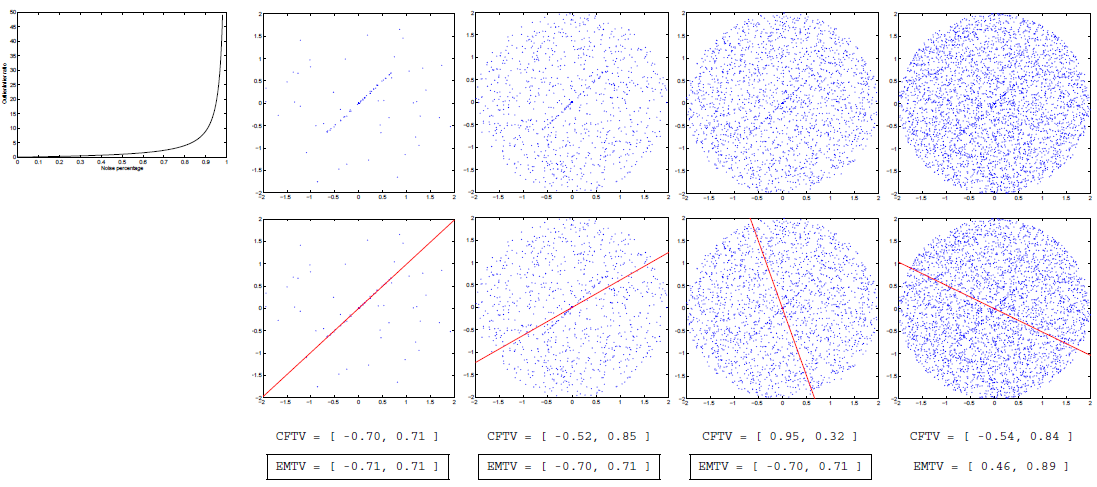}
\caption{\small
The top-left subfigure shows the plot of $\frac{R}{R+1}$. 
The four 2D data sets shown here have OI ratios $[1,20,45,80]$ 
respectively, which correspond to
outlier percentages $[50\%,95\%,98\%,99\%]$. Our EMTV can tolerate
OI ratios $\le$ 51 in this example. The original input,
the estimated line after the first EMTV iteration 
using CFTV to initialize the algorithm, and the line
parameters after the first EMTV iteration and final EMTV convergence
were shown.  The ground-truth parameter is $[-0.71, 0.71]$.
}
\label{fig:percent_vs_IO}
\end{center}
\vspace{-0.2in}
\end{figure*}

\subsection{Implementation and Initialization}

In summary, Eqns~(\ref{eqn:cal_S_inv}),~(\ref{eqn:e_step_update})
and~(\ref{eqn:m_step_update}) are all the equations needed to
implement EMTV and therefore the implementation is straightforward.

Noting that initialization is important to an EM algorithm, to initialize EMTV, we set $\sigma_1$
to be a very large value, $\bK_i = \mathbf{I}$
and $w_i = 1$ for all $i$.  $\bS_{ij}'$ is initialized to
be the inverse of $\bS_{ij}$, computed using the closed-form
solution presented in the previous section.
These initialization values mean that at the beginning we have no preference for the surface orientation. So all the input points are initially considered as inliers. With such initialization, we execute the first and the second rules in
Eqn.~(\ref{eqn:m_step_update}) in sequence. Note that when the first rule is being executed, the term involving $\bv$ is ignored because of the large $\sigma_1$, thus we can obtain $\bK_i^{-1}$
for the second rule. After that, we can start executing the
algorithm from the E-step.  This initialization procedure
is used in all the experiments in the following sections.
Fig.~\ref{fig:percent_vs_IO} shows the result after the first 
EMTV iteration on an example; note in particular that even though the initialization is 
at times not close to the solution
our EMTV algorithm can still converge to the desired ground-truth 
solution.

\section{Experimental Results}
\label{sec:results}
First, quantitative comparison
will be studied to evaluate EMTV with well-known algorithms:
RANSAC~\cite{fischler_acm81},
ASSC~\cite{wang_and_suter}, and TV~\cite{book_with_names}.
In addition, we also provide the result using the least squares method
as a baseline comparison.
Second, we apply our method to real data
with synthetic outliers and/or noise where the
ground truth is available, and perform comparison.
Third, more experiments on multiview stereo matching on
real images are performed.

As we will show, EMTV performed the best in highly corrupted data, because
it is designed to seek one linear structure of known type (as opposed to
multiple, potentially nonlinear structures of unknown type).  The use
of orientation constraints, in addition to position constraints, makes
EMTV superior to the random sampling methods as well.

\noindent {\bf Outlier/inlier (OI) ratio} \ 
We will use the {\em outlier/inlier (OI) ratio} to characterize
the outlier level, which is related to the outlier percentage
$Z \in [0,1]$
\begin{equation}
\label{eqn:percent_vs_IO}
Z = \frac{R}{R + 1}
\end{equation}
where $R$ is the OI ratio.
Fig.~\ref{fig:percent_vs_IO} shows a plot of $Z = \frac{R}{R+1}$
indicating that it is much more difficult for a given method to handle
the same percentage increase in outliers as the value of $Z$ increases.
Note the rapid increase in the number of outliers as $Z$ increases from
50\% to 99\%. That is, it is more difficult for a given method 
to tolerate an addition of, say 20\% outliers, when $Z$ is
increased from 70\% to 90\% than from 50\% to 70\%.
Thus the OI ratio gives more insight in studying an 
algorithm's performance on severely corrupted data.

\vspace{-0.1in}
\subsection{Robustness}
We generate a set of 2D synthetic data
to evaluate the performance on line fitting, by randomly sampling $44$ points from a line within the range $[-1,-1] \times [1, 1]$
where the locations of the points are contaminated by Gaussian noise of $0.1$ standard deviation. 
Random outliers were added to the data with different OI ratios.

The data set is then partitioned into two:

\begin{itemize}
\item {\sc Set 1}: OI ratio $\in [0.1, 1]$ with step size 0.1,
\item {\sc Set 2}: OI ratio $\in [1,100]$ with step size 1.
\end{itemize}

In other words, the partition is done at 50\% outliers. 
Note from the plot in Fig.~\ref{fig:percent_vs_IO}
that the number of outliers increases rapidly after 50\% outliers.  Sample data sets with
different OI ratios are shown in the top of Fig.~\ref{fig:percent_vs_IO}. Outliers were
added within a bounding circle of radius 2. In particular, the bottom of
Fig.~\ref{fig:percent_vs_IO} shows the result of the first EMTV
iteration upon initialization using CFTV.
%that is, each data point 
%is associated with the majority direction $\hat{e}_1$ of the 
%eigensystem given by $\bK$, multiplied by $w$ in (\ref{eqn:e_step_update}) 
%given in the first EMTV iteration.

\begin{figure*}
\begin{center}
\includegraphics[width=0.99\linewidth]{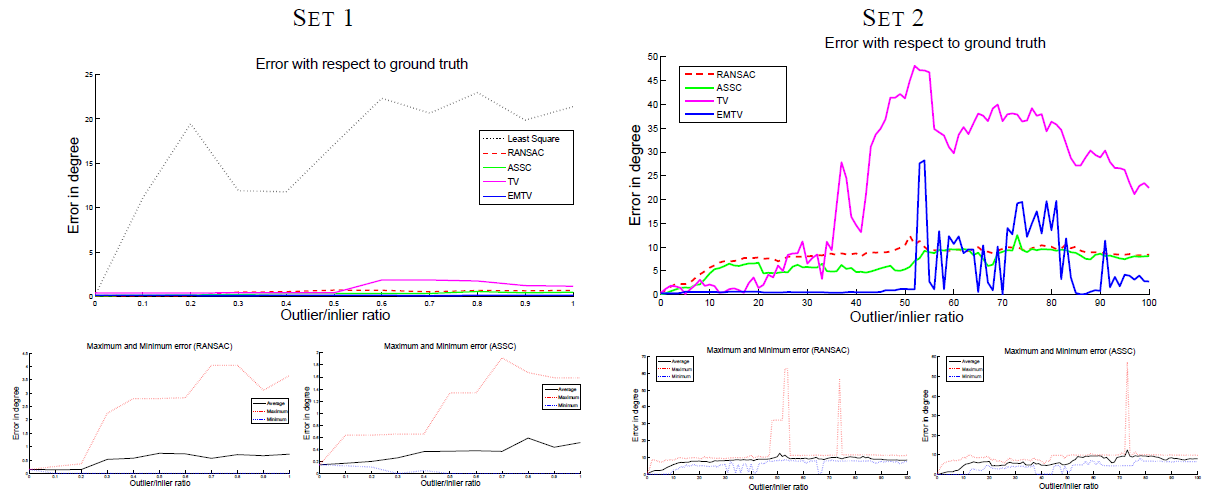} 
\end{center}
\caption{\small Error plots for {\sc Set 1} (OI ratio $=[0.1, 1]$, up to 50\% outliers)
and {\sc Set 2}  (OI ratio $=[1,100]$, $\ge$ 50\% outliers). Left:
for {\sc Set 1}, all the tested methods except the least-squares demonstrated reliable results. %except least-square.
EMTV is deterministic and converges quickly, capable of
correcting Gaussian noise inherent in the inliers and
rejecting spurious outliers, and resulting in the almost-zero
error curve.
Right: for {\sc Set 2}, EMTV still has an almost-zero error curve
up to an OI ratio of 51 ($\simeq$ 98.1\% outliers).
We ran 100 trials in RANSAC and ASSC and averaged the results.
The maximum and minimum errors of RANSAC and ASSC
are shown below each error plot.
}
\label{fig:plot_small_noise}
\label{fig:plot_large_noise}
\vspace{-0.1in}
\end{figure*}

The input scale, which is used in RANSAC, TV and EMTV, was estimated automatically
by TSSE proposed in~\cite{wang_and_suter}.  Note in principle these scales are
not the same, because TSSE estimates the scales of residuals in the normal space. Therefore,
the scale estimated by TSSE used in TV and EMTV are only
approximations. As we will demonstrate below, even with such rough
approximations, EMTV still performs very well showing that it is not
sensitive to scale inaccuracy, a nice property of tensor voting
which will be shown in an experiment to be detailed shortly.  
Note that ASSC~\cite{wang_and_suter} does not require any input scale.

\begin{figure*}[t]
\begin{center}
\includegraphics[width=0.85\linewidth]{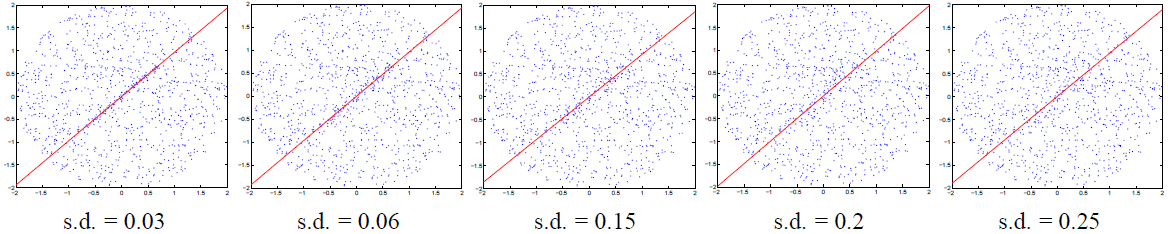}
\end{center}
\caption{\small Inputs containing various measurement errors, with OI ratio = 10 and fixed
outliers location. The estimated models (depicted by the red lines) obtained using EMTV are overlayed on
the inputs. Notice the line cluster becomes less salient when s.d. = 0.25. }
\label{fig:varying_sd_input}
\vspace{-0.2in}
\end{figure*}

{\sc Set 1} -- Refer to the {\em left} of
Fig.~\ref{fig:plot_small_noise} which shows the error produced
by various methods tested on {\sc Set 1}. The error is measured by
the angle between the estimated line and the ground-truth.
Except the least squares  method,
we observe that all the tested methods (RANSAC, ASSC, TV and
EMTV) performed very well with OI ratios $\leq 1$. For RANSAC and
ASSC, all the detected inliers were finally used in parameter
estimation. Note that the errors measured for RANSAC and ASSC were
the average errors in 100 executions\footnote{We
executed the algorithm 100 times. In each execution, iterative 
random sampling was done where the desired probability of
choosing at least one sample free from outliers was set to $0.99$
(default value). 
%The source code was obtained from P. Kovesi's website
%http://www.csse.uwa.edu.au/$\sim$pk/Research/MatlabFns/
}, 
%which are supposed to ease the random nature of the two robust methods based on random sampling. 
Fig.~\ref{fig:plot_small_noise} also shows the
maximum and minimum errors of the two methods after running 100
trials.  EMTV does not have such maximum
and minimum error plots because it is deterministic.

Observe that the errors produced by our method are almost zero in {\sc Set 1}.
EMTV is deterministic and converges quickly, capable of
correcting Gaussian noise inherent in the inliers and
rejecting spurious outliers, and resulting in the almost-zero
error curve. RANSAC and ASSC have error $< 0.6^\circ$ , which is
still very acceptable.

\begin{figure}[t]
\begin{center}
\includegraphics[width=0.95\linewidth]{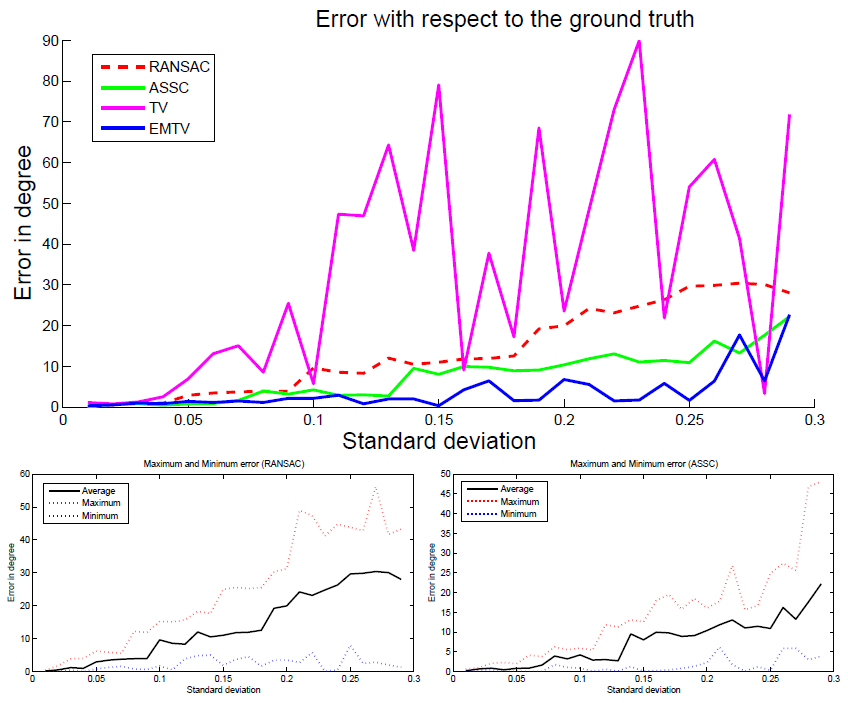} 
\end{center}
\caption{\small Measurement error: standard deviation varies from 0.01 to 0.29 with OI ratio at 10.  }
\label{fig:varying_sd_results}
\vspace{-0.2in}
\end{figure}

{\sc Set 2 } -- Refer to the {\em right} of Fig.~\ref{fig:plot_large_noise} which shows the result for {\sc Set 2}, from which we can distinguish the performance of the methods. TV breaks down at OI ratios $\geq 20$. After that, the performance of TV is unpredictable. EMTV breaks down at OI ratios $\ge$ 51, showing greater robustness than TV in this experiment due to the EM parameter fitting procedure.

The performance of RANSAC and ASSC were quite stable where the average errors
are within 4 and 7 degrees over the whole spectrum of OI ratios considered.
The maximum and minimum errors are shown in the bottom of
Fig.~\ref{fig:plot_large_noise}, which shows that they can be very large at
times. EMTV produces almost zero errors with OI ratio $\leq 51$, but then
breaks down with unpredictable performance.
From the experiments on {\sc Set 1} and {\sc Set 2} we conclude that
EMTV is robust up to an OI ratio of 51 ($\simeq$98.1\% outliers).

\noindent {\bf Insensitivity to choice of scale}.
%%Fig.~\ref{fig:plot_sensitivity} shows the result
We studied the errors produced by EMTV with different scales
$\sigma_d$ (Eqn.~(\ref{eqn:neighbor_weighting_Gaussian})), given
OI ratio of 10 ($\simeq$91\% outliers). Even in the presence of many
outliers, EMTV broke down only when $\sigma_d \simeq 0.7$
(the ground-truth $\sigma_d$ is 0.1), which indicates that our method is
not sensitive to large deviations of scale. Note that the scale parameter
can sometimes be automatically estimated (e.g., by modifying the original TSSE
to handle tangent space) as was done in the previous experiment.

\noindent {\bf Large measurement errors}.
In this experiment, we increased the measurement error by increasing the
standard deviation (s.d.) from 0.01
to 0.29, while keeping OI ratio equal to 10 and the location of the
outliers fixed. Some of the input data sets are depicted
in Fig.~\ref{fig:varying_sd_input}, showing that the inliers
are less salient as the standard deviation (s.d.) increases. A similar experiment was also
performed in~\cite{meer_bookchapter}.
Again, we compared our method with
RANSAC, ASSC and TV.
%where the outlier setting was the same as
%the aforementioned 2D line fitting experiments.

According to the error plot in the {\em top} of
Fig.~\ref{fig:varying_sd_results},
%the performance of least-squares is very stable,
%where the angular error is around 17 degrees across
%the whole s.d. spectrum. This is because the number
%of outliers is 10 times more than the inliers thus
%dominating the estimation result. On the other hand,
TV is very sensitive to the change of s.d.: when the
s.d. is greater than 0.03, the performance is unpredictable.
With increasing s.d., the performance of RANSAC and
ASSC degrade gracefully while ASSC always outperforms
RANSAC.  The {\em bottom} of Fig.~\ref{fig:varying_sd_results}
shows the corresponding maximum and minimum error
in 100 executions.

On the other hand, we observe the performance of EMTV
(with $\sigma_d = 0.05$) is extremely steady and accurate
when s.d. $<0.15$. After that, although its error plot
exhibits some perturbation, the errors produced are
still small and the performance is quite stable compared
with other methods.

\begin{figure}[t]
\begin{center}
\includegraphics[width=0.8\linewidth]{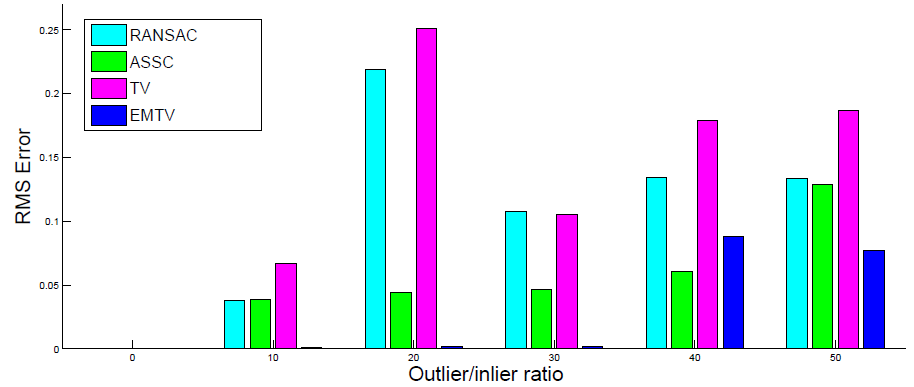}
\end{center}
\vspace{-0.1in}
\caption{\small {\em Corridor}. 
RMS error plot of various methods. }
\label{fig:epipolar_line}
\vspace{-0.15in}
\end{figure}

\begin{figure*}[t]
\begin{center}
\includegraphics[width=0.85\linewidth]{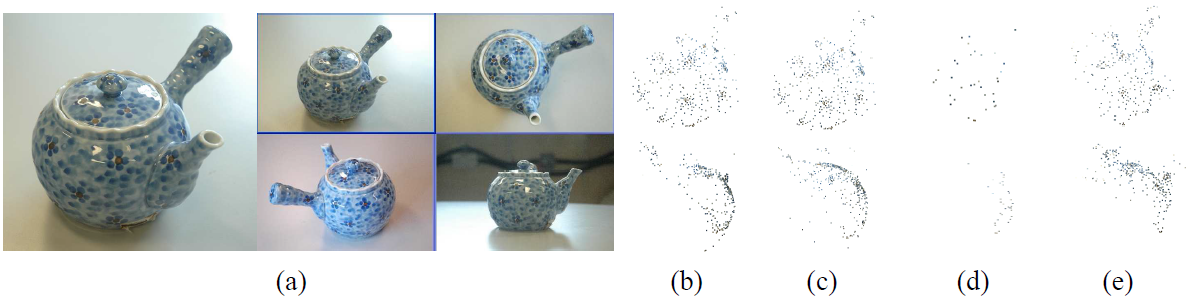}
\end{center}
\vspace{-0.1in}
\caption{\small
 {\em Teapot}: (a) 4 images (one in enlarged view) from the input image
 set consisting of 30 images captured around the object in a casual manner.
 (b)--(f) show two views of the sparse reconstruction generated by using 
{\tt KeyMatchFull} (398 points), 
{\tt linear\_match} (493 points),
{\tt ransac\_match} (37 points), 
{\tt assc\_match} (208 points), and
{\tt emtv\_match} (2152 points). 
 The candidate matches returned by SIFT are extremely noisy due to the
 ambiguous patchy patterns.
 On average 17404 trials were run in {\small {\tt ransac\_match}}.
 It is time consuming to run more trials on this noisy and large
 input where an image pair can have as many as 5000 similar matches.
 Similarly for {\tt assc\_match} where additional running time is needed
 to estimate the scale parameter in each iteration. On the other hand,
 {\tt emtv\_match} does not require any random sampling.
}
\label{fig:teapot}
\vspace{-0.1in}
\end{figure*}

\subsection{Fundamental Matrix Estimation}
\label{sec:example2}

%In this section, we will show how the 8-point algorithm can be
%modeled and solved by our method. This is a good example to
%demonstrate our method on high dimensional data.
%The underlying surface spanned in the feature space, which
%is a $9$-D hyperplane, is much simpler than that in the
%implicit surface modeling examples.

Given an image pair with $p \ge 8$ correspondences
${\cal P} = \{ (\mathbf{u}_i, \mathbf{u'}_i) | 8 \leq i \leq p \} $,
the goal is to estimate the
$3 \times 3$ fundamental matrix $\mathbf{F}= [f]_{a,b}$,
where $a,b \in \{ 1,2,3 \}$, such that
\begin{equation}
\label{eqn:essential} \mathbf{u'}_i^T \mathbf{F} \mathbf{u}_i = 0
\end{equation}
for all $i$. $\mathbf{F}$ is of rank 2. Let $\mathbf{u} = (u,v,1)^T$ and
$\mathbf{u'} = (u',v',1)$, Eqn.~(\ref{eqn:essential}) can be rewritten
into:
\begin{equation}
\label{eqn:essential_rearrange} \mathbf{U}_i^T \mathbf{h} = 0
\end{equation}
where
\begin{eqnarray}
\nonumber \mathbf{U} &=& (uu', uv', u, vu', vv' v, u', v', 1)^T \\
\nonumber \mathbf{v} &=& (f_{11}, f_{21}, f_{31}, f_{12}, f_{22}, f_{32}, f_{13}, f_{23}, f_{33} )^T
\end{eqnarray}

Noting that Eqn.~(\ref{eqn:essential_rearrange}) is a simple plane
equation, if we can detect and handle noise and outliers in the feature space,
Eqn.~(\ref{eqn:essential_rearrange}) should enable us to produce
a good estimation.
%Since $\mathbf{F}$ is defined up to a scaling factor, to avoid the trivial
%solution $\mathbf{h} = 0$, many existing methods impose the hard
%constraint that $ ||\mathbf{h}|| = 1 $. In EMTV, no special treatment
%is needed: this constraint is automatically  given by the second rule
%in Eqn.~(\ref{eqn:m_step_update}) where eigen-decomposition is used to solve
%for $\mathbf{v}$. 
Finally, we apply~\cite{hartley_defence} to obtain 
a rank-2 fundamental matrix. Data normalization is similarly
done as in~\cite{hartley_defence} before the optimization.

%let $\hat{\mathbf{F}}$ be the estimated matrix directly given by the
%optimized $\hat{\mathbf{h}}$, that is, the 9-D vector.
%and let, the SVD = \mathbf{H} \mathit{diag}(d_1, d_2, d_3) \mathbf{H}^T $
%be the SVD of $\mathbf{F}$, we set $d_3 = 0$. Similar to~\cite{hartley_defence},
%data normalization is done before the optimization.

We evaluate the results by estimating the fundamental matrix
of the data set {\em Corridor}, which is available at
www.robots.ox.ac.uk/$\sim$vgg/data.html. The matches
of feature points (Harris corners) are available. Random outliers
were added in the feature space. 

Fig.~\ref{fig:epipolar_line} shows the plot of RMS error, which
is computed by summing up and averaging
$\sqrt{\frac{1}{p}\sum_i|| \mathbf{U}_i^T \hat{\mathbf{h}}||^2}$
over all pairs,
where $\mathbf{U}_i$ is the set of clean data, and $\hat{\mathbf{h}}$
is the 9D vector produced from the rank-2 fundamental matrices
estimated by various methods. Note that all the images available
in the {\em Corridor} data set are used, that is, all $C^{11}_2$
pairs were tested.
%,and a few epipolar lines computed using EMTV are
%shown on a sample image pair are shown.
It can be observed
that RANSAC breaks down at an OI ratio $\simeq 20$, or 95.23\% outliers.
ASSC is very stable with RMS error $< 0.15$. TV breaks down at
an OI ratio $\simeq 10$. EMTV has negligible RMS error before it
starts to break down at an OI ratio $\simeq 40$. This finding
echoes that of~\cite{hartley_defence} that linear solution is sufficient
when outliers are properly handled.

\begin{figure}[t]
\begin{center}
\includegraphics[width=0.75\linewidth]{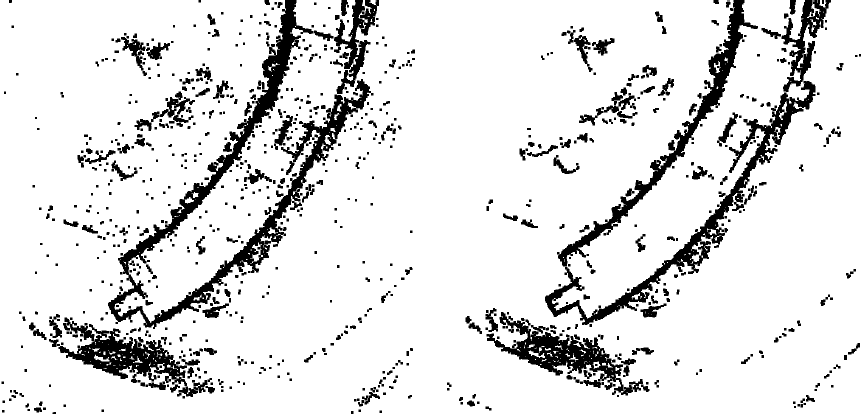}
\end{center}
\caption{\small Results before and after filtering of {\em Hall 3} (images shown
in Fig.~\ref{fig:hall3_recon}). 
%Top view of the reconstructed building
%is shown here. 
All salient 3D structures are retained in the filtered
result, including the bushes near the left facade and planters near the
right facade in this top view of the building.}
\label{fig:filtering}
\vspace{-0.2in}
\end{figure}

\begin{figure*}[t]
\begin{center}
\includegraphics[width=0.9\linewidth]{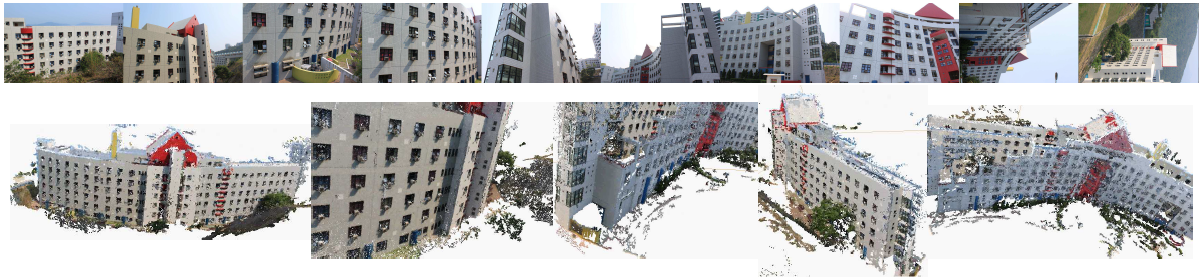}
\end{center}
\vspace{-0.2in}
\caption{\small The {\em Hall 3} reconstruction: ten of the input images (top) and five views of the quasi-dense 3D reconstruction (bottom).}
\label{fig:hall3_recon}
\vspace{-0.2in}
\end{figure*}

\subsection{Matching}
\label{sec:teapot}
In the uncalibrated scenario, EMTV estimates parameter accurately by employing
CFTV, and effectively discards epipolar geometries induced by wrong matches.
Typically, camera calibration is performed using nonlinear least-squares
minimization and bundle adjustment~\cite{lourakis_ba} which requires good
matches as input.  In this experiment, candidate matches
are generated by comparing the resulting 128D SIFT feature vectors~\cite{sift}, 
so many matched keypoints are not corresponding.

The epipolar constraint is enforced in
the matching process using EMTV, which returns the fundamental
matrix {\em and} the probability $w_i$ (Eqn~(\ref{eqn:e_step_update}))
of a keypoint pair $i$ being an inlier. In the experiment,
we assume keypoint pair $i$ is an inlier if $w_i > 0.8$.
\if 0
Note that we did not use any random sampling.
The following compares {\small {\tt emtv\_match}} with
{\small {\tt KeyMatchFull}}~\cite{snavely_ijcv07},
{\small {\tt assc\_match}}~\cite{wang_and_suter},
{\small {\tt ransac\_match}}, and
{\small {\tt linear\_match}},
where the latter three perform hyper-plane fitting by using
ASSC, RANSAC, and least-squares respectively.

\begin{figure}[t]
\begin{center}
\begin{tabular}{@{\hspace{0mm}}c@{\hspace{0mm}}c@{\hspace{0mm}}c}
\includegraphics[width=0.30\linewidth]{figure/LongJing/TeaBoxUnd2.bmp.eps} &
\includegraphics[width=0.32\linewidth]{figure/LongJing/original_tea.bmp.eps} &
\includegraphics[width=0.32\linewidth]{figure/LongJing/EM_tea.bmp.eps} \\
{\small (a) } & {\small (b)} & {\small (c)} \\
\end{tabular}
\end{center}
\caption{\small {\em Tea Can}: (a) One of the two input images. (b) Sparse
reconstruction generated by using {\small {\tt KeyMatchFull}}. (c) Sparse
reconstruction generated by using {\small {\tt emtv\_match}}. }
\label{fig:longjing}
\vspace{-0.2in}
\end{figure}

\noindent {\em Tea Can}.  Fig.~\ref{fig:longjing} shows
that, by using our filtered matches, even in the absence
of any focal length input, our sparse reconstruction of
the tea can (the image pair was obtained from~\cite{zzhang_pami}),
produced by the nonlinear least-squares minimization and
bundle adjustment~\cite{lourakis_ba}, is denser and contains
less errors as compared with~\cite{snavely_ijcv07}, where
we can faithfully reconstruct the right-angled container.

\noindent {\em Teapot}. 
\fi
Fig.~\ref{fig:teapot} shows
our running example {\em teapot} which contains
repetitive patterns across the whole object. Wrong matches
can be easily produced by similar patterns on different parts
of the teapot. This data set contains 30 images captured
using a Nikon D70 camera. Automatic configuration was set
during the image capture.
Visually, the result produced using {\small {\tt emtv\_match}}
is much denser than the results produced with
{\small {\tt KeyMatchFull}}~\cite{snavely_ijcv07}, 
{\small {\tt linear\_match}},
{\small {\tt assc\_match}}~\cite{wang_and_suter}, 
and {\small {\tt ransac\_match}}.
Note in particular that only {\small {\tt emtv\_match}}
recovers the overall geometry of the teapot, whereas
the other methods can only recover one side of the teapot.

This example is challenging, because the teapot's shape is quite symmetric
and the patchy patterns look identical everywhere. 
As was done in~\cite{snavely_ijcv07},
each photo was paired respectively with a number of photos with
camera poses satisfying certain basic criteria conducive to matching
or making the numerical process stable
(e.g. wide-baseline stereo). We can regard this pair-up process
as one of computing connected components. If the fundamental matrix
between any successive images are incorrectly estimated,
the corresponding components will no longer be connected,
resulting in the situation that only one side or part of the
object can be recovered.

Since {\small {\tt KeyMatchFull}} and {\small {\tt linear\_match}}
use simple distance measure for finding matches, the coverage of
the corresponding connected components tend to be small. 
It is interesting to note that the worst result is produced by using
{\small {\tt ransac\_match}}. This can be attributed to three reasons:
(1) the fundamental matrix is of rank 2 which implies that $\bv$
spans a subspace $\leq$ 8-D rather than a 9-D hyperplane;
(2) the input matches contain too many outliers for some image pairs;
(3) it is not feasible to fine tune the scale parameter for
every possible image pair and so we used a single value for all of the images.
A slight improvement could be found from
ASSC. However, it still suffers from problems (1) and (2) and so the
result is not very good even compared with
{\small {\tt KeyMatchFull}} and {\small {\tt linear\_match}}.

%While {\small {\tt KeyMatchFull}} can still handle this data set,
%we observe that many outliers {\em and} inliers were rejected
%as well. This is because {\small {\tt KeyMatchFull}} employed a
%restrictive criterion to drastically reduce the number
%of outliers. Specifically, they used $d_1 < 0.6 d_2$,
%where $d_1$ and $d_2$ are respectively
%the shortest and second shortest distance between a point
%and a candidate match in the 128D feature space. In other words,
%many similar structures or
%repeated patterns were filtered out, and only very distinctive
%feature pairs were retained for the following bundle adjustment
%stage. It is interesting that the performance of {\small {\tt KeyMatchFull}}
%and {\small {\tt linear\_match}} are quite similar.

On the other hand, {\small {\tt emtv\_match}} utilizes the epipolar
geometry constraint by computing the fundamental matrix in a data
driven manner. Since the outliers are effectively filtered out,
the estimated fundamental matrices are sufficiently accurate
to pair up all of the images into a single connected component. 
Thus, the overall 3D geometry can be recovered from all the available views.

\if 0
\begin{figure*}[!ht]
  \begin{center}
  \begin{tabular}{@{\hspace{-0.05in}}c@{\hspace{1mm}}c@{\hspace{1mm}}c@{\hspace{1mm}}c@{\hspace{1mm}}c}
  \includegraphics[width=0.14\linewidth]{figure/earth/earth_view1.bmp.eps} &
  \includegraphics[width=0.14\linewidth]{figure/earth/earth2_pmvs_normal_view0.bmp.eps} &
  \includegraphics[width=0.14\linewidth]{figure/earth/earth2_tmvs_normal_view0.bmp.eps} &
  \includegraphics[width=0.17\linewidth]{figure/earth/earth2_pmvs_view0.bmp.eps} &
  \includegraphics[width=0.17\linewidth]{figure/earth/earth2_tmvs_view0.bmp.eps} \\
(a) & (b) & (c) & (d) & (e)
  \end{tabular}
  \end{center}
  \caption{\small {\em Earth}. (a) an input image (81 in total), (b) and (c): zoom-in
   views of the normal reconstruction produced by PMVS and TMVS after the propagation
   step, (d) and (e) show respectively one view of the quasi-dense reconstruction by
  PMVS and TMVS.}
  \label{fig:earth}
\vspace{-0.1in}
\end{figure*}
\fi

\subsection{Multiview Stereo Reconstruction}
\label{sec:mvs}
This section outlines how CFTV and EMTV are applied to 
improve the match-propagate-filter pipeline in multiview stereo.
Match-propagate-filter is a competitive approach to 
multiview stereo
reconstruction for computing a (quasi) dense representation. Starting
from a sparse set of initial matches with high confidence, matches
are propagated using photoconsistency to produce a (quasi) dense
reconstruction of the target shape. Visibility consistency can be
applied to remove outliers.
Among the existing works using the match-propagate-filter
approach, patch-based multiview stereo (or PMVS) proposed
in~\cite{furukawa_pami09} has produced some best results to date. 

We observe that PMVS had not fully utilized the 3D information
inherent in the sparse and dense geometry before, during 
and after propagation, as patches do not adequately communicate 
among each other. As noted in~\cite{furukawa_pami09}, data communication 
should not be
done by smoothing, but the lack of communication will cause
perturbed surface normals and patch outliers during the propagation 
stage.  In~\cite{tmvs}, we proposed tensor-based multiview stereo (TMVS) 
and used 3D structure-aware tensors which communicate among each other 
via CFTV.  We found that such tensor communication not only improves 
propagation in MVS without undesirable
smoothing but also benefits the entire match-propagate-filter pipeline
within a unified framework. 

We captured 179 photos around a building which 
were first calibrated as described in section~\ref{sec:teapot}.
All images were taken on the ground
level not higher than the building, so we have very few samples of the
rooftop. The building facades are curved and the windows on the building
look identical to each other. The patterns on the front and back facade
look nearly identical. These ambiguities cause significant challenges
in the matching stage especially
for wide-baseline stereo. TMVS was run to obtain the quasi-dense reconstruction,
where MRFTV was used to filter outliers as shown in Fig.~\ref{fig:filtering}.
Fig.~\ref{fig:hall3_recon} shows the 3D reconstruction which is faithful
to the real building.
Readers are referred to~\cite{tmvs} for more detail and experimental
evaluation of TMVS.

\if 0

\begin{figure}[t]
\begin{center}
\begin{tabular}{
@{\hspace{0mm}}c@{\hspace{0mm}}c
@{\hspace{0mm}}c@{\hspace{0mm}}c
}
\includegraphics[height=1.5in]{figure/tripp/tripp_image.eps} &
\includegraphics[height=1.5in]{figure/tripp/tripp_pmvs_view0.bmp.eps} &
\includegraphics[height=1.5in]{figure/tripp/tripp_tmvs_view0.bmp.eps} &
\end{tabular}
\end{center}
\vspace{-0.1in}
\caption{\small {\em Tripp} reconstruction from sparse data set:
three input images (left) and the quasi-dense 3D reconstruction
produced by PMVS (middle) and TMVS (right).}
\label{fig:tripp}
%\vspace{-0.2in}
%\end{figure}
%\begin{figure}[t]
\begin{center}
\includegraphics[width=0.99\linewidth]{figure/george/george_image.eps}
\begin{tabular}{
@{\hspace{0.2mm}}c@{\hspace{0.2mm}}c
@{\hspace{0.2mm}}c@{\hspace{0.2mm}}c
@{\hspace{0.2mm}}c@{\hspace{0.2mm}}c
@{\hspace{0.2mm}}c@{\hspace{0.2mm}}c
@{\hspace{0.2mm}}c@{\hspace{0.2mm}}c
}
\includegraphics[height=1in]{figure/george/george_view0.bmp.eps} &
\includegraphics[height=1in]{figure/george/george_view1.bmp.eps} &
\includegraphics[height=1in]{figure/george/george_view3.bmp.eps} &
\includegraphics[height=1in]{figure/george/george_view2.bmp.eps}
\end{tabular}
\end{center}
\vspace{-0.1in}
\caption{\small {\em George} reconstruction from sparse data set:
five input images (top) and four views of the quasi-dense
3D reconstruction (bottom).}
\label{fig:george}
\vspace{-0.1in}
\end{figure}
\fi

\section{Conclusions}
A closed-form solution is proved for the special theory of tensor voting
(CFTV) for computing an exact structure-aware tensor in any dimensions.
For structure propagation, we derive a quadratic energy for MRFTV, thus 
providing a convergence proof for tensor voting which is impossible 
to prove using the original tensor voting procedure. Then, 
we derive EMTV for optimizing both the tensor and model parameters for 
robust parameter estimation. We performed quantitative and qualitative 
evaluation using challenging synthetic and real data sets.  
%Second, in tensor-based 
%multiview stereo (TMVS), we demonstrated how CFTV, EMTV and MRFTV benefit 
%the entire match-propagate-filter pipeline in multiview stereo 
%reconstruction.  
In the future we will develop a closed-form solution for the general theory of tensor voting, and extend EMTV to extract 
multiple and nonlinear structures.  We have provided C$++$ source code, 
but it is straightforward to implement 
Eqns~(\ref{eqn:cal_S}), (\ref{eqn:cal_S_inv}), 
(\ref{eqn:e_step_update}), (\ref{eqn:m_step_update}), 
(\ref{eqn:energy_update_rule}), and (\ref{eqn:sor_update_rule}).
We demonstrated promising results in multiview stereo, and will 
apply our closed-form solution to address important computer 
vision problems.

\ifCLASSOPTIONcompsoc
%  % The Computer Society usually uses the plural form
  \section*{Acknowledgments}
\else
  \section*{Acknowledgment}
\fi

The authors would like to thank the Associate Editor and all of the
anonymous reviewers. Special thanks go to Reviewer 1 for his/her 
helpful and detailed comments throughout the review cycle. 
%This work 
%was supported by the Hong Kong Research Grant Council (grant numbers 619711, 
%412911), the Google Faculty Award, and .... 

\ifCLASSOPTIONcaptionsoff
  \newpage
\fi



{\small
\bibliographystyle{ieee}
\bibliography{cftv}
}

\vspace{-0.5in}

\begin{IEEEbiography}
[{\includegraphics[width=1\linewidth]{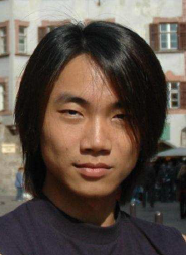}}] 
{Tai-Pang Wu} 
received the PhD degree in computer science from the Hong
Kong University of Science and Technology (HKUST) in 2007. He was
awarded the Microsoft Fellowship in 2006. After graduation, he has
been employed as a postdoctoral fellow in Microsoft Research Asia
Beijing (2007--2008) and the Chinese University of Hong Kong
(2008--2010). He is currently a Senior Research at the Enterprise and 
Consumer Electronics Group of the Hong Kong Applied Science 
and Technology Research Institute (ASTRI). 
His research interests include computer vision and computer graphics. 
He is a member of the IEEE and the IEEE Computer
Society.
\end{IEEEbiography}

\vspace{-0.5in}
\begin{IEEEbiography} 
[{\includegraphics[width=1\linewidth]{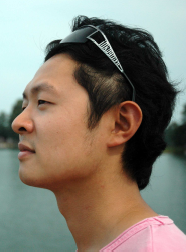}}] 
{Sai-Kit Yeung} received his PhD degree in electronic and computer
engineering from the Hong Kong University of Science and Technology
(HKUST) in 2009. He received his BEng degree (First Class Honors) in
computer engineering and MPhil degree in bioengineering from HKUST in
2003 and 2005 respectively. 
%He was a visiting student at Image %Sciences Institute, University Medical Center Utrecht, Utrecht, 
%The Netherlands in 2007 summer and the Image Processing Research Group at
%UCLA during 2008 spring and summer. 
He is currently an Assistant Professor at the Singapore University 
of Technology and Design (SUTD).
Prior to joining SUTD, he was a Postdoctoral Scholar in the Department
of Mathematics, University of California, Los Angeles (UCLA) in 2010.
His research interests include computer vision, computer graphics,
computational photography, image/video processing, and medical
imaging. He is a member of IEEE and IEEE Computer Society.
\end{IEEEbiography}

%\vspace{-1.0in}
\begin{IEEEbiography}
[{\includegraphics[width=1\linewidth]{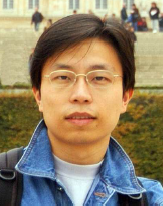}}] 
{Jiaya Jia} 
Jiaya Jia received the PhD degree in computer science from Hong Kong
University of Science and Technology in 2004 and is currently an associate
professor in Department of Computer Science and Engineering at the Chinese
University of Hong Kong (CUHK). He was a visiting scholar at Microsoft
Research Asia from March 2004 to August 2005 and conducted collaborative
research at Adobe Systems in 2007. He leads the research group in CUHK,
focusing on computational photography, 3D reconstruction, practical
optimization, and motion estimation. He serves as an associate editor for
TPAMI and as an area chair for ICCV 2011. He was on the program committees
of several major conferences, including ICCV, ECCV, and CVPR, and co-chaired
the Workshop on Interactive Computer Vision in conjunction with ICCV 2007.
He received the Young Researcher Award 2008 and Research Excellence Award
2009 from CUHK. He is a senior member of the IEEE.
\end{IEEEbiography}

\vspace{-1.0in}

\begin{IEEEbiography}
[{\includegraphics[width=1\linewidth]{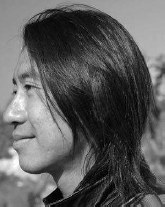}}] 
{Chi-Keung Tang}
received the MSc and PhD degrees in
computer science from the University of Southern California
(USC), Los Angeles, in 1999 and 2000, respectively.
Since 2000, he has been with the Department of
Computer Science at the Hong Kong University of Science and
Technology (HKUST) where he is currently a professor. He is an 
adjunct researcher at the Visual Computing Group of Microsoft
Research Asia.  His research areas are computer vision, computer
graphics, and human-computer interaction. He is an associate editor
of IEEE Transactions on Pattern Analysis and Machine Intelligence (TPAMI), 
and on the editorial board of International Journal of Computer Vision (IJCV).
He served as an area chair for ACCV 2006 (Hyderabad), ICCV 2007
(Rio de Janeiro), ICCV 2009 (Kyoto), ICCV 2011 (Barcelona), 
and as a technical papers committee member for the inaugural 
SIGGRAPH Asia 2008 (Singapore), SIGGRAPH 2011 (Vancouver), and 
SIGGRAPH Asia 2011 (Hong Kong), SIGGRAPH 2012 (Los Angeles).  
He is a senior member of the IEEE and the IEEE Computer Society.
\end{IEEEbiography}

\vspace{-1.0in}
\begin{IEEEbiography}
[{\includegraphics[width=1\linewidth]{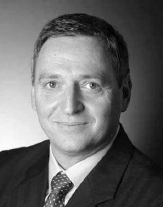}}] 
{G\'{e}rard Medioni} 
received the Dipl\^{o}me d'Ing\'{e}nieur
from the \'{E}cole Nationale Sup\'{e}rieure des
T\'{e}l\'{e}communications (ENST), Paris, in 1977 and
the MS and PhD degrees from the University of
Southern California (USC) in 1980 and 1983,
respectively. He has been with USC since then
and is currently a professor of computer science
and electrical engineering, a codirector of the
Institute for Robotics and Intelligent Systems
(IRIS), and a codirector of the USC Games
Institute. He served as the chairman of the Computer Science
Department from 2001 to 2007. He has made significant contributions
to the field of computer vision. His research covers a broad spectrum of
the field, such as edge detection, stereo and motion analysis, shape
inference and description, and system integration. He has published
3 books, more than 50 journal papers, and 150 conference articles. He
is the holder of eight international patents. He is an associate editor of
the Image and Vision Computing Journal, Pattern Recognition and
Image Analysis Journal, and International Journal of Image and Video
Processing. He served as a program cochair of the 1991 IEEE
Computer Vision and Pattern Recognition (CVPR) Conference and
the 1995 IEEE International Symposium on Computer Vision, a general
cochair of the 1997 IEEE CVPR Conference, a conference cochair of
the 1998 International Conference on Pattern Recognition, a general
cochair of the 2001 IEEE CVPR Conference, a general cochair of the
2007 IEEE CVPR Conference, and a general cochair of the upcoming
2009 IEEE CVPR Conference. He is a fellow of the IEEE, IAPR, and
AAAI and a member of the IEEE Computer Society.
\end{IEEEbiography}

\ifCLASSOPTIONcompsoc
  % IEEE Computer Society needs nocompress option
  % requires cite.sty v4.0 or later (November 2003)
  % \usepackage[nocompress]{cite}
\else
  % normal IEEE
  % \usepackage{cite}
\fi
% cite.sty was written by Donald Arseneau
% V1.6 and later of IEEEtran pre-defines the format of the cite.sty package
% \cite{} output to follow that of IEEE. Loading the cite package will
% result in citation numbers being automatically sorted and properly
% "compressed/ranged". e.g., [1], [9], [2], [7], [5], [6] without using
% cite.sty will become [1], [2], [5]--[7], [9] using cite.sty. cite.sty's
% \cite will automatically add leading space, if needed. Use cite.sty's
% noadjust option (cite.sty V3.8 and later) if you want to turn this off.
% cite.sty is already installed on most LaTeX systems. Be sure and use
% version 4.0 (2003-05-27) and later if using hyperref.sty. cite.sty does
% not currently provide for hyperlinked citations.
% The latest version can be obtained at:
% http://www.ctan.org/tex-archive/macros/latex/contrib/cite/
% The documentation is contained in the cite.sty file itself.
%
% Note that some packages require special options to format as the Computer
% Society requires. In particular, Computer Society  papers do not use
% compressed citation ranges as is done in typical IEEE papers
% (e.g., [1]-[4]). Instead, they list every citation separately in order
% (e.g., [1], [2], [3], [4]). To get the latter we need to load the cite
% package with the nocompress option which is supported by cite.sty v4.0
% and later. Note also the use of a CLASSOPTION conditional provided by
% IEEEtran.cls V1.7 and later.

% *** GRAPHICS RELATED PACKAGES ***
%
\ifCLASSINFOpdf
  % \usepackage[pdftex]{graphicx}
  % declare the path(s) where your graphic files are
  % \graphicspath{{../pdf/}{../jpeg/}}
  % and their extensions so you won't have to specify these with
  % every instance of \includegraphics
  % \DeclareGraphicsExtensions{.pdf,.jpeg,.png}
\else
  % or other class option (dvipsone, dvipdf, if not using dvips). graphicx
  % will default to the driver specified in the system graphics.cfg if no
  % driver is specified.
  % \usepackage[dvips]{graphicx}
  % declare the path(s) where your graphic files are
  % \graphicspath{{../eps/}}
  % and their extensions so you won't have to specify these with
  % every instance of \includegraphics
  % \DeclareGraphicsExtensions{.eps}
\fi
\hyphenation{op-tical net-works semi-conduc-tor}

%\begin{document}

%\title{Addendum to ``A Closed-Form Solution to Tensor Voting: Theory and Applications"}

\author{
        Tai-Pang~Wu,~%~\IEEEmembership{Fellow,~OSA,}
	Sai-Kit~Yeung,~%~\IEEEmembership{Fellow,~IEEE,}
        Chi-Keung~Tang$^*$\thanks{Please direct all inquiries to the contact author. E-mail: cktang@cse.ust.hk},
%~%~\IEEEmembership{Fellow,~OSA,}

        Jiaya Jia~%~\IEEEmembership{Fellow,~OSA,}
        and~Gerard~Medioni%~\IEEEmembership{Life~Fellow,~IEEE}% <-this % stops a space
\if 0
\IEEEcompsocitemizethanks{ 
\IEEEcompsocthanksitem T.-P.~Wu and C.-K.~Tang are with the Department of Computer Science
and Engineering, Hong Kong University of Science and Technology, Clear Water Bay,
Hong Kong. Email: \{pang,cktang\}@cs.ust.hk. 
\IEEEcompsocthanksitem S.-K. Yeung is with the Pillar of Information Systems
Technology and Design, Singapore University of Technology and Design, Singapore. Email: saikit@sutd.edu.sg.
\IEEEcompsocthanksitem J.~Jia is with the Department of Computer Science 
and Engineering, Chinese University of Hong Kong. Email: leojia@cse.cuhk.edu.hk.
\IEEEcompsocthanksitem G.~Medioni is with the Department of Computer Science, 
University of Southern California. Email: medioni@iris.usc.edu.
}
\fi
}%

\IEEEcompsoctitleabstractindextext{%
\begin{abstract}
\boldmath
We respond to the comments paper~\cite{comment} on the proof to
the closed-form solution to tensor voting~\cite{cftv} or CFTV.
First, the proof is correct and let $\bS$ be the resulting tensor
which may be asymmetric.  Second, $\bS$ should be interpreted using
singular value decomposition (SVD), where the symmetricity of $\bS$
is unimportant, because the corresponding eigensystems
to the positive semidefinite (PSD) systems, namely, $\bS \bS^T$ or $\bS^T \bS$, 
are used in practice.  Finally, we prove a symmetric version of 
CFTV, run extensive simulations and show that the original 
tensor voting, the asymmetric CFTV and symmetric CFTV produce 
{\em practically the same}
empirical results in tensor direction except
in high uncertainty situations due to ball tensors and low saliency.
\end{abstract}
% not want either math or citations to appear in the abstract.

% Note that keywords are not normally used for peer review papers.
%\begin{keywords}
%Transparent object, normal reconstruction, graph-cuts, segmentation
%\end{keywords}
}

\maketitle

% is not selected <OR> if conference mode is selected - because compsoc
% conference papers position the abstract like regular (non-compsoc)
\IEEEdisplaynotcompsoctitleabstractindextext

% \fi
\IEEEpeerreviewmaketitle

\setcounter{section}{0}
\setcounter{theorem}{0}
\setcounter{footnote}{0}

\section*{Addendum}
A closed-form solution to tensor voting or CFTV was proved in~\cite{cftv}.  
With CFTV, discrete voting field is no longer required where uniform
sampling, computation and storage efficiency are issues in high dimensional 
inference.  

%%Here, we respond to~\cite{comment} on the correctness of CFTV.  

We respond to the comments paper~\cite{comment} on the proof to
the closed-form solution to tensor voting~\cite{cftv} or CFTV.
First, the proof is correct and let $\bS$ be the resulting tensor
which may be asymmetric.  Second, $\bS$ should be interpreted using
singular value decomposition (SVD), where the symmetricity of $\bS$
is unimportant, because the corresponding eigensystems
to the positive semidefinite (PSD) systems, namely, $\bS \bS^T$ or $\bS^T \bS$, 
are used in practice.  Finally, we prove a symmetric version of 
CFTV, run extensive simulations and show that the original 
tensor voting, the asymmetric CFTV and symmetric CFTV produce 
{\em practically the same}
empirical results in tensor direction except
in high uncertainty situations due to ball tensors and low saliency.

Dirty codes for the experimental section to show the practical equivalence 
of the asymmetric CFTV, symmetric CFTV, and original discrete tensor voting 
(that is, Eq.~(2)) in 2D are available\footnote{Reader may however find it easier 
to implement the closed form solutions on their own and generate the discrete voting 
fields for direct comparison; 2D codes for generating discrete voting fields using 
the equation right after ``So we have" and before the new integration by parts 
introduced in~\cite{cftv} are available upon request.}.

\section{Asymmetric CFTV}
We reprise here the main result in~\cite{cftv} in an equivalent form:
\begin{theorem}
\emph{(Closed-Form Solution to Tensor Voting)} 
\label{thm:cftv}
The tensor vote at $\bx_i$ induced by $\bK_j$ located at $\bx_j$ is given by
the following closed-form solution:
\begin{equation}
\bS_{ij} = c_{ij} \bR_{ij} \left( \bK_j - \frac{1}{2} \bK_j \br_{ij}\br_{ij}^T \right) \bR^T_{ij}
\label{eq:cftv}
\end{equation}
where $\bK_j$ is a second-order symmetric tensor,
$\bR_{ij} ={\mathbf I} - 2 {\mathbf r}_{ij} {\mathbf r}_{ij}^T$,
${\mathbf I}$ is an identity, 
${\mathbf r}_{ij}$ is a unit vector pointing from
$\bx_j$ to $\bx_i$ and $c_{ij} = \exp( - \frac{||\bx_i - \bx_j ||^2 }{ \sigma_d })$
with $\sigma_d$ as the scale parameter.
\end{theorem}
To simplify notation we will drop the subscripts in $\bR$ and $\br$, and
let $\bT= \frac{1}{2} \bK_j \br\br^T$.
%because in the rest of 
%the paper we will consider two sites: the vote receiver at $\bx_i$ and the 
%voter at $\bx_j$.  

In~\cite{comment} an example is given: for an input PSD $\bK_j = 
\left[
\begin{array}{cc}
\frac{1}{2} & 0 \\
0 & 1
\end{array}
\right]$, 
the output $\bS$ computed using Theorem~\ref{thm:cftv} 
%$\bS = 
%\frac{1}{8}
%\left[
%\begin{array}{cc}
%6 & -2 \\
%-1 & -3 
%\end{array}
%\right]$
is asymmetric, much less that $\bS$ is a PSD matrix.  It was also 
pointed out in~\cite{comment} potential technical flaws 
involving the derivative with respect to a unit stick tensor in 
the proof to Theorem~\ref{thm:cftv}, which is related to 
Footnote~3 in~\cite{cftv}.
%%on the contraction a sequence of identical unit tensors multiplied together.

Note that Theorem~\ref{thm:cftv} does not guarantee the output 
$\bS$ is symmetric or PSD. In the C codes accompanying~\cite{cftv}, which is 
available in the IEEE digital library, the {\tt eig\_sys} function 
behaves like singular value decomposition {\tt svd}.  The statements 
and experiments pertinent to $\bS$ in~\cite{cftv} subsequent to Theorem~\ref{thm:cftv} 
in fact refer to the SVD results on $\bS$, and we apologize for not 
making this explicitly clear in the paper.

Recall the singular value decomposition and the eigen-decomposition 
are related, namely, the left-singular vectors of $\bS$ are eigenvectors 
of $\bS \bS^T$, the right singular-vectors of $\bS$ are eigenvectors of 
$\bS^T \bS$, where the eigenvectors are orthonormal bases.  
%the non-zero singular values of $\bS$ are the square 
%roots of the non-zero eigenvalues of both $\bS \bS^T$ and $\bS^T \bS$.  
We performed sanity check by running extensive simulations in 
dimensions up to 51 and show that 
all eigenvalues of $\bS \bS^T$ and $\bS^T \bS$ are nonnegative,
and that they are PSD.  The symmetricity of $\bS$ is unimportant 
in practice. 

Nonetheless, for theoretical interest we provide in the following an alternative
proof to Theorem~\ref{thm:cftv} which produces a symmetric $\bS$ and
serves to dispel the flaws pointed out in~\cite{comment}.  
Finally we run simulations to show that the original tensor voting, 
the asymmetric CFTV in~\cite{cftv} and the symmetric 
CFTV in the following produce the practically the same results.
%in tensor direction except in high uncertainty situations due to ball
%tensors and low saliency. 

\section{Symmetric CFTV}
From the first principle, integrating unit stick tensors $\bN_{\theta} = \tn_{\theta} \tn_{\theta}^T$ 
in all directions $\theta$ with strength $\tau_\theta^2$, we obtain Eq.~(8) in~\cite{cftv}:

\begin{eqnarray}
\label{eqn:cf_inter} \bS_{ij} &=& c_{ij} \mathbf{R} \left(  \mathbf{K}_j - 
\int_{\mathbf{N}_{\theta} \in \nu} \tau_{\theta}^2 \mathbf{N}_{\theta} \mathbf{r} \mathbf{r}^T \mathbf{N}_{\theta} d\mathbf{N}_{\theta} \right) \mathbf{R}^T
\label{eq:eq8}
\end{eqnarray}
Note the similar form of Eqs~(\ref{eq:cftv}) and (\ref{eq:eq8}), and 
the unconventional integration domain in Eq.~(\ref{eq:eq8}) 
where $\nu$ represents the space of stick tensors given
by $\tn_{\theta j}$ as explained in~\cite{cftv}\footnote{The suggestion in~\cite{comment}
was our first attempt and as explained in~\cite{cftv}, it does not have obvious advantage
while making the derivation unnecessarily complicated.}.  Let $\Omega$ be
the integration domain to shorten our notations.
 
Let ${\bT}_{\mathit{sym}}$ be the integration in Eq.~(\ref{eqn:cf_inter}).  
$\bT_{\mathit{sym}}$ can be solved by integration by parts.  Here, we repeat 
Footnote~3 in~\cite{cftv} on the contraction of a sequence of identical unit 
stick tensors when multiplied together.  Let $\bN_{\theta} = \tn_{\theta} \tn_{\theta}^T$ be a unit stick tensor, where $\tn_{\theta}$ is a unit normal at angle $\theta$, and $q \ge 1$ be a positive integer, then
\begin{equation}
\bN_{\theta}^q = \tn_{\theta}  \tn_{\theta}^T  \tn_{\theta}  \tn_{\theta}^T \cdots  \tn_{\theta}  \tn_{\theta}^T =  \tn_{\theta} \cdot 1 \cdot 1 \cdots 1 \cdot  \tn_{\theta}^T =  \tn_{\theta}  \tn_{\theta}^T = \mathbf{N}_{\theta}.
\label{eq:eat}
\end{equation}
To preserve symmetry, we leverage Footnote~3 or Eq.~(\ref{eq:eat}) above but rather than exclusively using contraction, (i.e., $\bN_{\theta}^q = \bN_{\theta}$) as done in~\cite{cftv}, we use expansion (i.e., $\bN_{\theta} = \bN_{\theta}^q$) in the following:
\begin{eqnarray}
\bT_{\mathit{sym}} 
= \int_{\Omega} \tau_{\theta}^2 \mathbf{N}_{\theta} \mathbf{r} \mathbf{r}^T \mathbf{N}_{\theta} d\mathbf{N}_{\theta}
= \int_{\Omega} \tau_{\theta}^2 \mathbf{N}_{\theta}^2 \mathbf{r} \mathbf{r}^T \mathbf{N}_{\theta}^2 d\mathbf{N}_{\theta} 
%%\nonumber &=& \int_{\mathbf{N}_{\theta} \in \nu} f(\theta) g'(\theta) \d d\mathbf{N}_{\theta} 
\end{eqnarray}
Let $f(\theta) = \tau_{\theta}^2 \mathbf{N}_{\theta}^2$, then
$f'(\theta) = 2 \tau_{\theta}^2 \bN_\theta d\bN_\theta = 2 \tau_{\theta}^2 \bN_\theta^2 d\bN_\theta$ after
expansion.  Similarly, let
$g(\theta) = \frac{1}{2} \mathbf{r} \mathbf{r}^T \mathbf{N}_{\theta}^2$ and 
$g'(\theta) = \mathbf{r} \mathbf{r}^T \mathbf{N}_{\theta} d\bN_\theta =
\mathbf{r} \mathbf{r}^T \mathbf{N}_{\theta}^2 d\bN_{\theta}$ after
expansion.  
\footnote{This is in disagreement with the claims about the authors' 
Eq.~(7) in~\cite{comment}, which was in fact never used in their 
intention in our derivation in~\cite{cftv}.}
Note also ${\mathbf K}_j$, in the most general form, can be
expressed as $\int_{\Omega} 
\tau_{\theta}^2 \mathbf{N}_{\theta} d\mathbf{N}_{\theta}$.  So, we obtain
\begin{eqnarray}
\nonumber {\bT}_{\mathit{sym}} &=&\int_{\Omega} \tau_{\theta}^2 \mathbf{N}_{\theta}^2 \mathbf{r} \mathbf{r}^T \mathbf{N}_{\theta}^2 d\mathbf{N}_{\theta} \\
\nonumber &=& \left[ f(\theta) g(\theta) \right]_{\Omega} - \int_{\Omega} f'(\theta) g(\theta) d\mathbf{N}_{\theta} \\
\nonumber &=& \left[ \frac{1}{2} \tau_{\theta}^2 \mathbf{N}_{\theta}^2
\mathbf{r} \mathbf{r}^T \mathbf{N}_{\theta}^2
\right]_{\Omega}
- \int_{\Omega} \tau_{\theta}^2 \bN_\theta^2 \mathbf{r} \mathbf{r}^T \mathbf{N}_{\theta}^2 d\mathbf{N}_{\theta} \\
\nonumber &=&  \left[ \frac{1}{2} \tau_{\theta}^2 \mathbf{N}_{\theta}^2
\mathbf{r} \mathbf{r}^T \mathbf{N}_{\theta}^2
\right]_{\Omega}
- \bT_{\mathit{sym}} \\
%\end{eqnarray}
%\begin{eqnarray}
\nonumber &=& \frac{1}{4}  \left[ \tau_{\theta}^2 \mathbf{N}_{\theta}^2
\mathbf{r} \mathbf{r}^T \mathbf{N}_{\theta}^2
\right]_{\Omega} \\
\nonumber &=&  \frac{1}{4} \int_{\Omega} \left( \tau_{\theta}^2 \frac{d}{d{\mathbf N}_{\theta}}[{\mathbf N}_{\theta}^2] {\mathbf r} {\mathbf r}^T {\mathbf N}_{\theta}^2
+ \tau_{\theta}^2 {\mathbf N}_{\theta}^2 \frac{d}{d{\mathbf N}_{\theta}}[{\mathbf r} {\mathbf r}^T {\mathbf N}_{\theta}^2] \right) d{\mathbf N}_{\theta} \\
\nonumber
&=&  \frac{1}{4} \int_{\Omega} \left( \tau_{\theta}^2 \frac{d}{d{\mathbf N}_{\theta}}[{\mathbf N}_{\theta}] {\mathbf r} {\mathbf r}^T {\mathbf N}_{\theta} + \tau_{\theta}^2 {\mathbf N}_{\theta} \frac{d}{d{\mathbf N}_{\theta}}[{\mathbf r} {\mathbf r}^T {\mathbf N}_{\theta}] \right) d{\mathbf N}_{\theta} \\
\\
\label{eq:contract} 
%\nonumber &=&  \frac{1}{4}   \int_{\Omega} \left(  \br \br^T \tau_\theta^2 \bN_\theta^2 + \tau_\theta^2 \bN_\theta^2 \br \br^T \right)  d{\mathbf N}_{\theta} \\
\nonumber &=&  \frac{1}{4}   \int_{\Omega} \left(  \br \br^T \tau_\theta^2 \bN_\theta + \tau_\theta^2 \bN_\theta \br \br^T \right)  d{\mathbf N}_{\theta} \\
\label{eq:symcftv} &=& \frac{1}{4} \left( \br \br^T \bK_j + \bK_j \br \br^T \right).
\end{eqnarray}
Here, in the derivative with respect to a unit stick tensor $\bN_\theta$ 
along the tensor direction, the tensor magnitude $\tau_\theta^2$ can 
be legitimately regarded as a constant\footnote{Here is another disagreement 
with~\cite{comment}: the tensor direction and tensor magnitude are two entirely 
different entities.}. 

Note that we apply contraction in Eq.~(\ref{eq:contract}).  Comparing the
pertinent $\bT$ in the asymmetric CFTV in Eq.~(\ref{eq:cftv}) and symmetric CFTV in Eq.~(\ref{eq:symcftv}):
\begin{eqnarray*}
\bT_{\mathit{asym}} &=& \frac{1}{2} \bK_j \br \br^T \\ \nonumber  
\bT_{\mathit{sym}} &=& \frac{1}{4} \left( \br \br^T \bK_j + \bK_j \br \br^T \right)
\end{eqnarray*}
it is interesting to observe how tensor contraction and expansion when applied as
described can preserve tensor symmetry whereas in~\cite{cftv}, only tensor 
contraction was applied. 

Notwithstanding, the symmetricity of $\bS$ is unimportant in practice as 
singular value decomposition will be applied to $\bS$ before the 
eigenvalues and eigenvectors are used.  The following section reports 
the results of our experiments on the asymmetric and symmetric CFTV 
when used in practice. 

\section{Experiments}

In each of the following simulations, a random input PSD matrix 
$\bK = \bA \bA^T + \epsilon \bI$ is generated, where $\bA$ is 
a random square matrix, $\bI$ is an identity, $\epsilon$ is a 
small constant (e.g. $1e^{-2})$. 

\begin{enumerate}
\item  (Sanity check)
$N$D simulations ($N = 2$ to $51$) of $1000$ tensors for each dimension
\begin{enumerate}
\item The $\bS$ produced by symmetric CFTV is indeed symmetric. This is confirmed
by testing each $\bS$ using ${\mathit norm}(\bS^{-1} \bS^T-\bI) = 0$, 
where ${\mathit norm}(\cdot)$ is the $L_2$ norm of a matrix. 
\item $\bS^T \bS$ and $\bS \bS^T$ are PSD where $\bS$ can be produced by asymmetric or
symmetric CFTV.  This is validated by checking all eigenvalues being nonnegative.
\end{enumerate}

\item 
2D simulations of more than 1 million tensors show the practical equivalence 
in tensor direction among
\begin{enumerate}
\item discrete solution to original tensor voting Eq.~(\ref{eq:eq8}), or Eq.~(8) in~\cite{cftv},
\item $\bS^T \bS$ and $\bS \bS^T$ produced by asymmetric CFTV,
\item $\bS^T \bS$ ($ = \bS \bS^T$ when $\bS$ is symmetric) produced by symmetric CFTV,
\end{enumerate}
while relative tensor saliency is preserved.
\label{eqtest}

\item 
$N$D simulations ($N = 2$ to $51$) of $1000$ tensors for each dimension 
show the practical equivalence in tensor direction among
\begin{enumerate}
\item $\bS^T \bS$ and $\bS \bS^T$ produced by asymmetric CFTV,
\item $\bS^T \bS$ ($ = \bS \bS^T$) produced by symmetric CFTV,
\end{enumerate}
in their largest eigenvectors which encompass the most ``energy" while the rest
represents uncertainty in orientation each spanning a plane perpendicular to 
the largest eigenvectors.
\label{eqtest3}

\end{enumerate}

For simulations in (\ref{eqtest}), we exclude ball tensors from our tabulation 
for the obvious reason: any two orthonormal vectors describe the equivalent 
unit ball tensor.  The mean and maximum deviation in tensor
direction are respectively 0.9709 and 0.9537 (score for perfect alignment is 1) in terms 
of the dot product among the relevant eigenvectors.  
The deviation can be explained by the imperfect uniform sampling for a 2D ellipse used in 
computing the discrete tensor voting solution: there is no good way
for uniform sampling in $N$ dimensions\footnote{While uniform sampling on a 2D circle
is trivial, uniform sampling on a 2D ellipse is not straightforward.  
For a 3D {\em sphere}, recursive subdivision of an icosahedron is 
a good approximation but no good approximation for a 3D ellipsoid exists.}. 
When
we turned off the discrete tensor voting but compared only the asymmetric
and symmetric CFTV, the mean and maximum deviation in tensor direction 
are respectively improved to 0.9857 and 0.9718.   

For tensor saliency, we found that while the normalized saliencies are 
not identical among the four versions of tensor voting~\footnote{The 
tensor saliency produced by discrete simulation of Eq.~(\ref{eq:eq8}) 
is proportional to the number of samples, thus we normalize the eigenvalues
such that the smallest eigenvalue is 1.}, their relative order is preserved.  
That is, when we sort the eigenvectors according to their corresponding 
eigenvalues in the respective implementation of tensor voting, the 
sorted eigenvectors are always empirically identical among the four 
cases.
%%, noting one additional case for the asymmetric CFTV due to different $\bS \bS^T$ and $\bS^T \bS$. 

For simulations in (\ref{eqtest3}), we did not compare the discrete solution: 
in higher dimensions, uniform sampling of an $N$D ellipsoid is an issue to 
discrete tensor voting.  The mean and maximum deviation among the largest
eigenvector of the three versions are respectively 0.9940 and 0.9857. 
 
\section{Epilogue}
%%We welcome~\cite{comment} which gives us the opportunity to 
We revisit the main result in~\cite{cftv}, and reconfirm the efficacy of the 
closed-form solution to tensor voting which votes for the 
most likely connection without discrete voting fields, which is 
particularly relevant in high dimensional inference where uniform
sampling and storage of discrete voting fields are issues.  As an aside,
we prove the symmetric version of CFTV.

The application of tensor contraction and expansion given by Eq.~(\ref{eq:eat}) 
is instrumental to the derivation of the closed-form solution.  Interestingly,
while $\bS$ may be asymmetric, pre-multiplying or post-multiplying by itself 
not only echos the contraction/expansion operation given by Eq.~(\ref{eq:eat}) but also
produces a PSD system that agrees with the original tensor voting result.  
As shown above, the inherent flexibility also makes symmetric CFTV possible. 
Thus we believe further exploration may lead to useful and interesting 
theoretical results on tensor voting.
\ifCLASSOPTIONcaptionsoff
  \newpage
\fi



%\bibliographystyle{IEEEtran}

{\small

%\bibliographystyle{ieee}
%\bibliography{response}
}

\end{document}